\documentclass{article}

\usepackage[final]{neurips_2024}
\usepackage{graphicx}
\graphicspath{{Figs/}}
\usepackage{comment}
\usepackage{caption}
\usepackage[utf8]{inputenc} 
\usepackage[T1]{fontenc}  
\usepackage{hyperref}   
\usepackage{url}   
\usepackage{booktabs}   
\usepackage{amsfonts}      
\usepackage{nicefrac}     
\usepackage{microtype}     
\usepackage[linesnumbered,ruled,vlined,onelanguage,algo2e]{algorithm2e}
\usepackage{makecell}
\usepackage{xcolor}
\usepackage[framemethod=tikz]{mdframed}
\usepackage{mathtools}
\usepackage{amsmath,amsfonts,amssymb,amscd,xspace}
\usepackage{amsthm}
\usepackage[shortlabels]{enumitem}
\usepackage{bbm}
\usepackage{bm}
\usepackage{booktabs}
\usepackage{algorithm}
\usepackage{algpseudocode}
\usepackage{CJKutf8}

\graphicspath{{Figs/}}

\usepackage{graphicx}

\newtheorem{assumption}{Assumption}

\newcommand{\nt}{T}

\newcommand{\nn}{n}

\DeclareMathOperator*{\argmax}{arg\,max}

\newtheorem{theorem}{Theorem}
\numberwithin{theorem}{section}
\newtheorem{lemma}{Lemma}
\newtheorem{corollary}{Corollary}
\numberwithin{corollary}{section}

\newtheorem{definition}{Definition}

\begin{document}

% \title{Safe Exploitative Play in Bayesian Games with Untrusted Type Beliefs}

\title{Safe Exploitative Play with Untrusted Type Beliefs}

\author{
  Tongxin~Li\textsuperscript{\rm 1}\thanks{Correspondence to: Tongxin Li <\texttt{litongxin@cuhk.edu.cn}>.} ~~~~~~   Tinashe Handina\textsuperscript{\rm 2} ~~~~~~ Shaolei Ren\textsuperscript{\rm 3}
  ~~~~~~   Adam Wierman\textsuperscript{\rm 2}\\
  \\
  \textsuperscript{\rm 1}School of Data Science \\
  The Chinese University of Hong Kong, Shenzhen, China \\
  \texttt{litongxin@cuhk.edu.cn} \\
\\
  \textsuperscript{\rm 2}\text{Computing + Mathematical Sciences} \\
  California Institute of Technology, Pasadena, USA \\
  \texttt{ \{thandina, adamw\}@caltech.edu}\\
\\ 
  \textsuperscript{\rm 3}Electrical \& Computer Engineering \\
  University of California, Riverside, USA \\
  \texttt{shaolei@ucr.edu}
}

% \author{
%   Tongxin~Li\thanks{Correspondence to: Tongxin Li <\texttt{litongxin@cuhk.edu.cn}>.} \\
%   School of Data Science \\
%   The Chinese University of Hong Kong, Shenzhen, China \\
%   \texttt{litongxin@cuhk.edu.cn}
%   \And 
%    Tinashe Handina \\
%   Computing + Mathematical Sciences \\
%   California Institute of Technology, Pasadena, USA \\
%   \texttt{thandina@caltech.edu}
%   \And
%   Shaolei Ren \\
%   Electrical \& Computer Engineering \\
%   University of California, Riverside, USA \\
%   \texttt{shaolei@ucr.edu}
%   \And
%   Adam Wierman \\
%   Computing + Mathematical Sciences \\
%   California Institute of Technology, Pasadena, USA \\
%   \texttt{adamw@caltech.edu}
% }

\maketitle

\begin{abstract}
The combination of the Bayesian game and learning has a rich history, with the idea of controlling a single agent in a system composed of multiple agents with unknown
behaviors given a set of types, each specifying a possible behavior for the other agents. The idea is to plan an agent's own actions with respect to those types which it believes are most likely to maximize the payoff. However, the type beliefs are often learned from past actions and likely to be incorrect. With this perspective in mind, 
we consider an agent in a game with type predictions of other components, and investigate the impact of incorrect beliefs to the agent’s payoff. In particular, we formally define a tradeoff between risk and opportunity by comparing the payoff obtained against the optimal payoff, which is represented by a gap caused by trusting or distrusting the learned beliefs. 
Our main results
characterize the tradeoff by establishing upper and lower bounds on the Pareto front for both normal-form and stochastic Bayesian games, with numerical results provided. 
\end{abstract}

\section{Introduction}

% {\color{blue}
% \begin{itemize}
%     \item State the converse first.
% \end{itemize}}

% The quote above from John F. Kennedy intriguingly captures the dual nature of danger, or risk and opportunity (known as\begin{CJK*}{UTF8}{bsmi}
% 危機
% \end{CJK*} in Chinese). This duality resonates deeply within the realm of real-world multi-agent systems, where risk and opportunity often coexist. 

\begin{quote}
“\textit{The Chinese symbol for crisis is composed of two elements: one signifies danger and the other opportunity.}”
— Lewis Mumford, 1944
\end{quote}

The famous interpretation of the Chinese word for ‘crisis’ (known as\begin{CJK*}{UTF8}{bsmi}
危機\end{CJK*}), although based on mistaken etymology, captures the dual nature of danger and opportunity. It  provides an interesting metaphor for the inherent complexities within real-world multi-agent systems. 
These systems, where risk and opportunity are often inextricably linked, play a pivotal role across diverse domains, ranging from human-AI collaboration \cite{yan2024efficient} and cyber-physical systems \cite{wang2016towards}, to highly competitive environments like real-time strategy games \cite{vinyals2019grandmaster} and poker \cite{brown2019superhuman}.

In conventional applied and theoretical frameworks, it is typically assumed that all agents either cooperate or adhere to pre-defined policies~\cite{lowe2017multi, albrecht2018autonomous,rashid2020monotonic, anagnostides2024convergence}. However, real-world scenarios often defy these simplifications, presenting agents that display a spectrum of behaviors ranging from cooperative, to heterogeneous, irrational, or even adversarial~\cite{li2023byzantine}. This deviation from expected behavior patterns complicates the dynamics of multi-agent systems, as agents cannot reliably predict the actions of their counterparts. The uncertainty regarding whether to trust or distrust predictions naturally leads to a critical tradeoff between the coexisted risk and opportunity, reflecting the dual aspects highlighted in Mumford’s remark.

% Multi-agent systems  are essential in diverse contexts, from human-AI collaboration~\cite{yan2024efficient} and cyber-physical systems~\cite{wang2016towards, zhang2021multi} to competitive domains such as real-time strategy games~\cite{vinyals2019grandmaster} and poker~\cite{brown2019superhuman} where players interact in a shared environment. However, in contrast to assumptions in both applied and theoretical frameworks that all players are either cooperative or implement pre-defined policies~\cite{lowe2017multi, albrecht2018autonomous,rashid2020monotonic, anagnostides2024convergence}, real-world applications often involve players that have heterogeneous, irrational, or even adversarial behaviors~\cite{li2023byzantine}. Furthermore,  players may not be fully trustworthy and thus do not know the behaviors of the others.

As critical examples, Bayesian games~\cite{harsanyi1967games,harsanyi1968games} provide an approach for modeling differing types of agents in strategic environments.  In these games, players form beliefs about others' types and update beliefs in
response to observed actions and choose their actions accordingly. For example, in competitive settings like poker or even in simple games like matching pennies, deviating from the game theoretic optimal (GTO) strategy~\cite{friedman1971optimal,zadeh1977computation} to exploit weaker opponents can be beneficial, but this approach also relies on potentially flawed type beliefs~\cite{albrecht2018autonomous}, making it risky to take advantage of such side-information. Similarly, in real-world problems like security games, having a prior distribution of attackers’ behavioral types in a Bayesian game setting leads to advantages. However, exploiting incorrect types can be risky compared to just using minimax strategies. 
Despite the ubiquity of incorrect type beliefs in practical scenarios, limited attention has been paid to explore such a tradeoff, with exceptions in designing heuristically safe and exploitative strategies in specific contexts, such as with Byzantine adversaries~\cite{li2023byzantine} and sequential games~\cite{milec2021continual}.

Inaccuracies in beliefs about others' types may arise from factors such as using out-of-distribution data to generate priors, changes in opponents' behaviors, or mismatches between hypothesized types and the optimal type space, etc. As noted in~\cite{albrecht2016belief}, it is evident that prior beliefs significantly affect the long-term performance of type-based learning algorithms like the Harsanyi-Bellman Ad Hoc Coordination (HBA). 
Although algorithms like HBA demonstrate asymptotic convergence to correct predictions or type distributions through various methods of estimating posterior beliefs, and methods exist for detecting inaccuracies in type beliefs using empirical behavioral hypothesis testing~\cite{albrecht2016belief}, the theoretical capabilities of general algorithms remain uncertain. 

In general, relying on learned beliefs presents a fundamental tradeoff between the potential payoffs from exploitative play and the risk of incorrect type beliefs. Given the potential inaccuracies in these type beliefs, relying on them to exploit opponents could lead to high-risk strategies. Conversely, not exploiting these beliefs might result in overly cautious play. Therefore, it is natural to investigate the impact of incorrect beliefs on the agent's payoff, in terms of a tradeoff between trusting or distrusting the beliefs of types provided by type-based learning algorithms (e.g., Bayesian learning~\cite{jordan1991bayesian}, best response dynamics~\cite{bichler2023computing}, and policy iteration with neural networks~\cite{bichler2021learning}, etc.) in multi-agent systems. To summarize the focus of this paper, we aim to address the following critical question:

\vspace{-1pt}

\begin{center}
    \textit{What is the fundamental tradeoff between trusting/distrusting type beliefs in games?}
\end{center}

\vspace{-1pt}

\textbf{Contributions.} 
Motivated by the above question, we analyze the following \textit{payoff gap} that arises from erroneous type beliefs: 
\begin{align}
\label{eq:informal_gap}
\Delta(\varepsilon;\pi) &\coloneq \max_{d\left(\theta,\theta^{\star}\right)\leq \varepsilon} \left(\max_{\phi} \mathsf{Payoff}(\phi,\theta^{\star}) - \mathsf{Payoff}(\pi(\theta),\theta^{\star})\right),
\end{align}
where $\theta$ and $\theta^{\star}$ are predicted and true type beliefs respectively; $d(\cdot,\cdot)$ measures the distance between $\theta$ and $\theta^{\star}$ bounded from above by  $\varepsilon$ and will be formally specified with concrete model contexts in Section~\ref{sec:nfg} and Section~\ref{sec:sbg}, together with the payoff, denoted by $\mathsf{Payoff}(\cdot,\cdot)$ as a function of the true type $\theta^{\star}$ and the used strategy $\pi$. The agent uses a strategy $\pi$ that depends on the type belief $\theta$. Overall, the payoff difference~\eqref{eq:informal_gap} above quantifies the worst-case gap between the optimal payoff (obtained using an optimal strategy that maximizes $\mathsf{Payoff}(\cdot,\theta^{\star})$), and the payoff corresponding to $\pi$. 

In particular, when the agent's strategy $\pi$ trusts the belief $\theta$, it takes the opportunity to close the gap in ~\eqref{eq:informal_gap} when the belief error $\varepsilon$ is small. However, when $\varepsilon$ increases, such a strategy incurs a high risk since it trusts the incorrect $\theta$.
Evaluating the payoff gap in~\eqref{eq:informal_gap} naturally yields a tradeoff between opportunity and risk, as illustrated on the right of Figure~\ref{fig:system}. 
To be more precise, we measure the tradeoff between two important quantities:

\vspace{-2pt}
\textit{(Missed) Opportunity}: 
$\Delta(0;\pi)$, corresponding to the case when the type beliefs are correct; 

\vspace{-4pt}
\textit{Risk}: $\max_{\varepsilon>0}\Delta(\varepsilon;\pi)$ measuring the payoff difference incurred by worst-case incorrect beliefs. 

\vspace{-2pt}
In summary, the (missed) opportunity measures the discrepancy of a strategy $\pi$ from the optimal strategy, which is aligned with the ground truth belief $\theta^{\star}$ of other players in terms of the obtained payoff. Additionally, the risk quantifies how inaccurate beliefs impact the difference in terms of payoffs in the worst case. The goal of this paper is to investigate safe and exploitative strategies in Bayesian games that achieve near-optimal opportunity and risk.

\begin{figure}[t]
\centering
\includegraphics[width=0.8\textwidth]{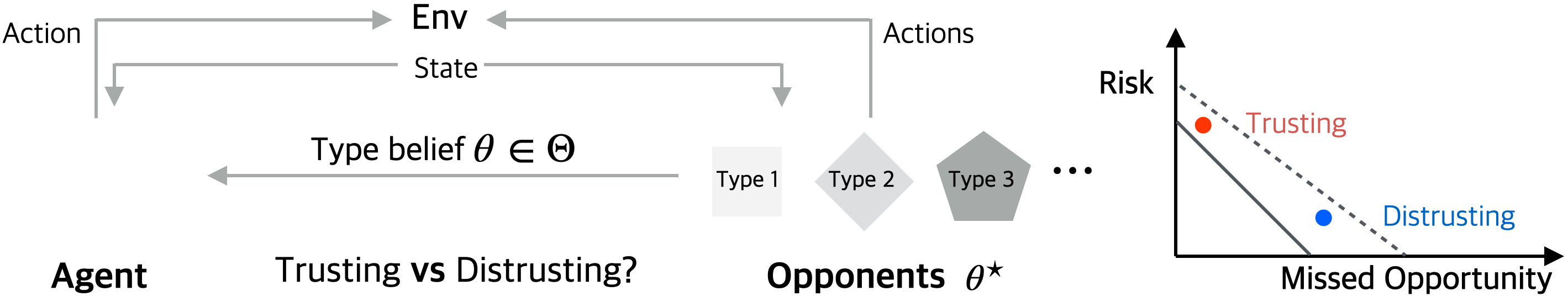}
% \caption{Illustration of the tradeoff in a normal-form game between opportunity and risk with $\alpha=5$ and $V\in \{1,2,\ldots,5\}$.}
\caption{\textbf{Left}: A stochastic Bayesian game where an agent interacts with an environment and opponents, with a belief of their types $\theta\in\Theta$. \textbf{Right}: The tradeoff between trusting and distrusting type beliefs, with trust leading to higher risk and opportunity and distrust resulting in lower risk and opportunity, implying an opportunity-risk tradeoff with varying strategy $\pi$.}
\label{fig:system}
\end{figure}

Our main results are two-fold. Firstly, in normal-form Bayesian games, we characterize a tradeoff between the opportunity and risk. We consider a strategy as a convex combination of a safe strategy and the best response given type beliefs. Upper bounds on opportunity and risk are provided in Theorem~\ref{thm:normal_form_upper}. Conversely, lower bounds that hold for any mixed strategy are shown in Theorem~\ref{thm:normal_form_lower}. Notably, when the game is fair, and the hypothesis set $\Theta$ is sufficiently large, these bounds tightly converge. Secondly, we explore a dynamic setting in stochastic Bayesian games, where an agent, provided with type beliefs about other players, engages in interactions over time, as illustrated in Figure~\ref{fig:system}. Unlike the normal-form approach, we utilize a value-based strategy that establishes upper bounds on opportunity and risk, as outlined in Theorem~\ref{thm:sbg_upper}. Additionally, Theorem~\ref{thm:sbg_lower} provides lower bounds on opportunity and risk that differ from the upper bounds by multiplicative constants, yielding a characterization of the opportunity-risk tradeoff. 
Finally, a case study of a security game, simulating a defender protecting an elephant population from illegal poachers, is provided in Section~\ref{sec:exp}.

% Suppose there are two players. The first player gets a prediction of the mixed strategy of the second player, denoted by $\widetilde{y}$. In the matrix general sum setting, the payoff matrix is denoted by $A\in \mathbb{R}^{m\times n}$ where the first player has $m$ choices of the rows and the second player has $n$ choices of the columns of $A$. The first player gets $a_{ij}$ and second player gets $-a_{ij}$.

% Poker GTO is a software that helps players move game theoretically optimal. 
% {Should I Always Follow GTO Strategy?}
% Professional poker player Daniel Negreanu advises that GTO poker should be your baseline strategy, but you should deviate from it in order to exploit your opponents. The key is to exploit your opponents with discreet adjustments so that they don’t realize your strategy.

% You must create a hybrid strategy between GTO and exploitative play to maximize your profit, but always remember that GTO play is not necessary against weak opponents. Weak opponents always make mistakes, which means that you must always target them with adjustments rather than worry about being defensive.

\section{Related Work}
\vspace{-5pt}

\textbf{Learning Types in Games.}
The concept of type-based methods dates back to the development of Bayesian games, as first established by Harsanyi in the late 1960s~\cite{harsanyi1967games,harsanyi1968games}. 
The concepts of learning and updating beliefs appear in pioneering works like the
adaptive learning~\cite{milgrom1991adaptive}. While much game theory research, including work by Kalai et al.~\cite{kalai1993rational}, focuses on equilibrium analysis through Bayesian belief-based learning, other studies, such as those by Nachbar et al.~\cite{nachbar1997prediction,nachbar2005beliefs} and~\cite{foster2001impossibility, dekel2004learning}, reveal the challenges players face in making correct predictions while playing optimally under certain game conditions and assumptions. 
From an application perspective, Southey et al.~\cite{southey2012bayes} applied type-based methods to poker, where players' hands are partially hidden. They demonstrated how to maintain and use beliefs to determine the best strategies in this setting.
Besides classic results exemplified above, closely related to our results, the recent work by Milec et al.~\cite{milec2021continual} addresses the limitations of Nash equilibrium strategies in two-player extensive-form games, particularly their inability to exploit the weaknesses of sub-optimal opponents. They defined the exploitability of a strategy as the expected payoff that a fully rational opponent can gain beyond the game's base value and introduced a method that ensures safety, defined as an upper limit on exploitability compared to the payoff obtained using a Nash equilibrium strategy. However, the proposed tradeoff between the exploitation of the opponent given the correct model and safety
against an opponent who can deviate arbitrarily from the predicted model is applicable specifically to an algorithm that employs a continual depth-limited restricted Nash response.

\textbf{Online Decision-Making with Predictions.} The tradeoff between opportunity and risk analyzed in this work is motivated by the recent progress in algorithms with predictions, also known as learning-augmented algorithms~\cite{mahdian2012online,purohit2018improving} for online decision-making problems such as caching~\cite{rohatgi2020near,lykouris2021competitive,im2022parsimonious}, bipartite matching~\cite{antoniadis2020online}, online optimization~\cite{christianson2022chasing,li2024robust}, control~\cite{li2021information,li2022robustness,lin2022bounded,li2023certifying}, valued-based reinforcement learning~\cite{golowich2022can,li2024beyond,yang2024anytime}, and real-world applications~\cite{li2021learning, christianson2022robustifying,li2023learning,li2024out}. This line of work investigates the impact of untrusted predictions on two key metrics known as consistency and robustness, which are defined based on the competitive ratio of the considered contexts. 
It is also worth highlighting that previous works in decision-making often provide best-of-both-worlds guarantees for both stochastic and adversarial environments~\cite{jin2021best,amir2022better}, while the results in this work shed light on studying the intermediate regimes that do not fully align with either stochastic or adversarial settings and delves into interactive environments formed by players whose behaviors may deviate from the type beliefs.

\textbf{Stochastic Bayesian Games.}
Our results presented in Section~\ref{sec:sbg} draw on foundational concepts from stochastic Bayesian games as outlined by Albrecht et al.~\cite{albrecht2015game,albrecht2016belief}, which merge concepts of the Bayesian games~\cite{harsanyi1967games,harsanyi1968games}
and the stochastic 
games~\cite{shapley1953stochastic}, with applications in cooperative multi-agent reinforcement
learning~\cite{li2023byzantine}. 
In~\cite{albrecht2016belief}, 
three methods—product, sum, and correlated—for integrating observed evidence into posterior beliefs have been explored. Specifically, it has been demonstrated that the Harsanyi-Bellman Ad Hoc Coordination (HBA) eventually make right future predictions under specific conditions using the product posterior beliefs. However, while HBA may eventually align with the correct type distribution, it is not guaranteed to learn it with the product posterior beliefs. Under certain conditions, HBA with the sum and correlated posterior converges to the correct type distribution. Nonetheless, even though methods like empirical behavioral hypothesis testing are available to detect inaccuracies in type beliefs, a comprehensive theoretical analysis is still lacking.

\vspace{-3pt}

% The stochastic elements of the SBG enable the consideration of uncertainty
% in the interactions of the attacker and defender. As in a stochastic game, a set of states define
% the environment for the players’ interactions, and the state transitions occur stochastically as
% a function of the players’ actions. The Bayesian elements of the SBG enable the consideration
% of the uncertainty regarding the attacker’s characteristics. As in a Bayesian game, a set
% of types are defined to describe possible behaviors of at least one of the players. 

% guesser’s advantage: ~\cite{eliaz2011edgar}

\vspace{-3pt}

\section{Normal-Form Bayesian Games with Untrusted Type Beliefs}
\label{sec:nfg}

\vspace{-5pt}

Let $\|\cdot\|_1$ denote the $\ell_1$-norm and $\|\cdot\|_{\max}$ denote the element-wise max norm. 

\vspace{-5pt}

\subsection{Problem Setting, Opportunity, and Risk}

\vspace{-5pt}

Consider a normal-form game with two players: Player 1 possesses a payoff matrix $A\in\mathbb{R}^{a\times b}$ with $\|A\|_{\max}\leq \alpha$, where the first player has $a$ choices of the rows and the second player has $b$ choices of the columns of $A$. Player 1 forms a belief about Player 2's mixed strategy, denoted by a probability distribution $\rho$ over a set of hypothesized strategies $\Theta$ that contains a ground truth strategy $y^{\star}$.\footnote{Note that we assume $y^{\star}\in\Theta$ for the ease of presentation aligning with Assumption 1 in~\cite{albrecht2016belief}, and our results can be easily generalized to the case when $\Theta$ does not contain $y^{\star}$ by modifying the definition of $\varepsilon$ to incorporate incomplete or incorrect hypothesized types.} 

As a concrete example of the \textit{payoff gap} proposed in~\eqref{eq:informal_gap}, the following benchmark characterizes the gap between payoffs obtained by a strategy with machine-learned belief $\rho$ and an optimal strategy knowing $y^{\star}$ beforehand:
\begin{align}
\label{eq:gap_normal_form}
\Delta_{\mathsf{NFG}}(\varepsilon;\pi) &\coloneq \max_{d\left(\rho,y^{\star}\right)\leq \varepsilon} \left(\max_{x\in\mathsf{P}_{a}}x^{\top}A y^{\star} - \pi(\rho)^{\top}A y^{\star}\right),
\end{align}
given a fixed policy $\pi:\mathsf{P}_\Theta\rightarrow \mathsf{P}_{a}$ that outputs a mixed strategy of the first player knowing $\rho$, where $d\left(\rho,y^{\star}\right)\coloneq \|\mathbb{E}_{\rho}[y]-y^{\star}\|_1$, $\mathsf{P}_\Theta$ and $\mathsf{P}_{a}$ are sets of probability distributions on $\Theta$ and the $a$ different choices of rows respectively. 

In particular, we focus on measuring a tradeoff between two important quantities, the \textit{(missed) opportunity} 
$\Delta_{\mathsf{NFG}}(0;\pi)$ and the \textit{risk}  $\max_{\varepsilon>0}\Delta_{\mathsf{NFG}}(\varepsilon;\pi)$. 
% In addition, a strategy $\pi$ is \textit{$\beta$-foregone} if  $\Delta_{\mathsf{NFG}}(0;\pi)\leq \beta(\pi)$ and \textit{$\delta$-risky} if $\max_{\varepsilon>0}\Delta_{\mathsf{NFG}}(\varepsilon;\pi)\leq \delta(\pi)$.
The former measures how far the considered strategy $\pi$ is away from the optimal strategy knowing the ground truth type $y^{\star}$ of Player 2 in terms of the payoff obtained; the latter quantifies the worst-case impact of inaccurate belief on the payoff difference. 

\vspace{-5pt}

\subsection{Motivating Example: Matching Pennies}
\label{sec:matching_pennies}

Before presenting general results, we illustrate the underlying concepts through a simple normal-form game. Consider the matching pennies game as a classic example, where the mixed strategies for players as defined follows. 

Suppose $a=b=2$ and each of the two players has the option to choose either Heads (\texttt{H}) or Tails (\texttt{T}). Let $y^{\star}\in [0,1]$ be the true probability that Player 2 plays \texttt{H} and $1-y^{\star}$ the probability of playing \texttt{T}. Suppose the hypothesis set $\Theta$ contains all possible mixed strategies, each corresponding to a type of Player 2. Then, Player 1 receives a belief of $y^{\star}$, denoted by $y$, and chooses a strategy $\pi$ that depends on $y$. Similarly, let $x=\pi(y)$ be the probability that Player 1 plays \texttt{H} and $1-x$ the probability of playing \texttt{T}. If two players' actions match, Player 1 will receive $1$ and $-1$ otherwise. 
The problem setting is summarized in Figure~\ref{fig:system_mp}. Given $y^{\star}$ and $x$, the expected payoff is therefore $(2y^{\star}-1)(2x-1)$. 
The best response of Player 1, depending on $y$, is characterized by:
\begin{align}
\label{eq:br_matching}
 \mathsf{BR}(y) = \begin{cases}
x = 0 & \text{if } y < \frac{1}{2}, \\
x \in [0,1] & \text{if } y = \frac{1}{2}, \\
x = 1 & \text{if } y > \frac{1}{2}. 
\end{cases}
\end{align}
Given a strategy $\pi:[0,1]\rightarrow [0,1]$, the payoff gap in~\eqref{eq:gap_normal_form} for this example is instantiated as
\begin{align}
\label{eq:gap_mp}
\Delta_{\mathsf{MP}}(\varepsilon;\pi) &\coloneq \max_{|y-y^{\star}|\leq \varepsilon} 2\Big( (2y^{\star}-1)\mathsf{BR}(y^{\star})- (2y^{\star}-1)\pi(y)\Big).
\end{align}

\begin{figure}[t]
\centering
\includegraphics[width=0.75\textwidth]{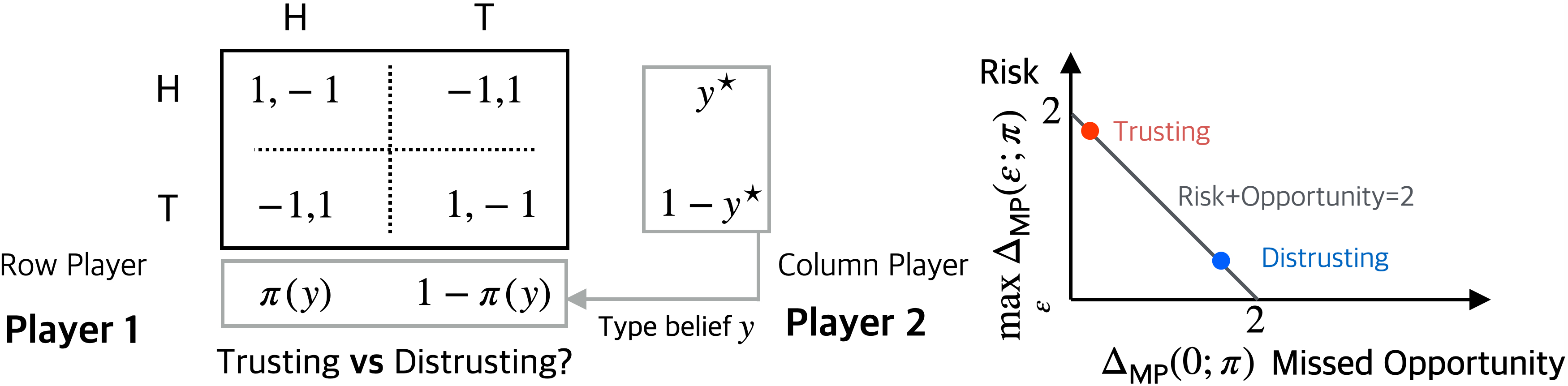}
\caption{\textbf{Left}: Matching Pennies payoff matrix for Player 1 (row player) with type belief $y$ and Player 2 (column player) whose strategy is defined by $y^{\star}$. \textbf{Right}: Opportunity-risk tradeoff that satisfies $\Delta_{\mathsf{MP}}(0;\pi)+\max_{\varepsilon}\Delta_{\mathsf{MP}}(\varepsilon;\pi)=2$.}
\label{fig:system_mp}
\end{figure}

\textbf{Impossibility.}
We first consider an arbitrary strategy $\pi:[0,1]\rightarrow [0,1]$ that misses at most $(1-\lambda)$ opportunity. Suppose $\pi$ chooses \texttt{H} with a probability $\pi(y)$ and \texttt{T} with a probability $1-\pi(y)$. Therefore, setting $\varepsilon=0$, the payoff gap must satisfy
$$
{\Delta_{\mathsf{MP}}(0;\pi) =\max_{y\in [0,1]} 2(2y-1)\left(\mathsf{BR}(y)-\pi(y)\right) \leq 1-\lambda},
$$
which implies $\pi\left(0\right)\leq (1-\lambda)/2$ and $\pi\left(1\right)\geq (1+\lambda)/2$ by setting $y=0$ and $y=1$, respectively.
Thus, plugging in $y^{\star}=1$ and $y=0$, we conclude that
\begin{align*}
% \label{eq:lower_bound_1}
\max_{\varepsilon}\Delta_{\mathsf{MP}}(\varepsilon;\pi) 
 \geq \max_{|y-y^{\star}|\leq 1} 2\left(\mathsf{BR}(y^{\star})-\pi(y)\right)\geq 2\left(\mathsf{BR}(1)-\pi(0)\right)\geq 1+\lambda,
\end{align*}
since $\mathsf{BR}(1)=1$ by~\eqref{eq:br_matching}. This implies that $\pi$ has at least $(1+\lambda)$ risk. Note that this argument holds for any arbitrarily chosen $\pi$. In a word, any strategy misses at most $1-\lambda$ opportunity must incur at least $1+\lambda$ risk, in terms of the payoff.

\textbf{Mixed Strategy Existence.}
Now, let us construct a concrete strategy to derive upper bounds on the (missed) opportunity and risk. The idea is to combine the strategy $\mathsf{BR}(y)$, which exploits the belief $y$ of types, and the Nash equilibrium strategy of this game, known as $\overline{x} = 1/2$, which is a solution of the minimax problem
$\min_{y\in [0,1]}\max_{x\in [0,1]} (2y-1)(2x-1)$ that enhances safety. 

Fix $\pi(y)\coloneq \lambda\mathsf{BR}(y)+(1-\lambda)\overline{x}$ a mixed strategy as a convex combination of the two strategies. 

The payoff gap in~\eqref{eq:gap_mp} satisfies
\begin{align*}
\Delta_{\mathsf{MP}}(\varepsilon;\pi) \leq\max_{|y-y^{\star}|\leq \varepsilon} 2(2y^{\star}-1) \Big(\lambda\left(\mathsf{BR}(y^{\star})-\mathsf{BR}(y)\right)+(1-\lambda)\left(\mathsf{BR}(y^{\star})-1/2\right)\Big).
\end{align*}

\textit{Opportunity:} Therefore, there exists a mixed strategy $\pi$ such that when $\varepsilon=0$ (i.e., $\mathsf{BR}(y)=\mathsf{BR}(y^{\star})$), its (missed) opportunity is bounded by
\begin{align*}
\Delta_{\mathsf{MP}}(0;\pi)=\max_{y\in [0,1]} 2(2y-1)\left(1-\lambda\right)\left(\mathsf{BR}(y)-1/2\right)\leq 1-\lambda.
\end{align*}

\textit{Risk:}
Moreover, maximizing over $\varepsilon$ such that $y=0$, $y^{\star}=1$, $\mathsf{BR}(y^{\star})-\mathsf{BR}(y)=1$, the risk for $\pi$ always satisfies $\max_{\varepsilon}\Delta_{\mathsf{MP}}(\varepsilon;\pi)\leq 1+\lambda$. 

In a word, we find a strategy that misses $1-\lambda$ opportunity and meanwhile has $1+\lambda$ risk.

\textbf{Pareto Optimality}
In conclusion, above construction shows that there is a mixed strategy $\pi$ for the matching pennies that misses $(1-\lambda)$ opportunity and incurs $(1+\lambda)$ risk. Conversely, any strategy that misses at most $(1-\lambda)$ opportunity must have at least $(1+\lambda)$ risk. The segment on the right of Figure~\ref{fig:system_mp} represents a Pareto front, and the strategy constructed as a convex combination is confirmed to be Pareto optimal following the arguments above.
This motivating example indicates a tight tradeoff between opportunity and risk for matching pennies.
In the sequel, we further generalize this result to normal-form games.

\subsection{Opportunity-Risk Tradeoff for Normal-Form Games}

In general, the opportunity-risk tradeoff depends on the hypothesis set $\Theta$ that contains a subset of candidate strategies and the payoff matrix of the game.
We state useful definitions to characterize properties of the hypothesis set $\Theta$ and the considered normal-form game. 

% Before we proceed to discuss our main results, we first consider a normal-form game as a warm-up example and formally state the results previewed in Figure~\ref{fig:normal_form} by constructing a specific strategy based on a belief distribution. Motivated by the construction of such a strategy, we present our analysis for stochastic Bayesian games in Section~\ref{sec:sbg}.

% Consider the following performance benchmark:

% \begin{align*}
%     \textsf{Utility-Gap}_{\pi}(\varepsilon)\coloneq \max_{A:\|A\|\leq \alpha} \max_{(y,\widetilde{y}):\|y-\widetilde{y}\|\leq \varepsilon} \left(\max_{x\in\mathsf{P}_a}x^{\top}A y - \pi(\widetilde{y})^{\top}A y\right)
% \end{align*}
% where $\pi:\mathsf{P}_b\rightarrow\mathsf{P}_a$ is a policy that outputs a mixed strategy of the first player.

\begin{definition}
\label{def:terms_nfg}
Given a hypothesis set $\Theta$, we let the \textit{diameter} of $\Theta$ with respect to the $\ell_1$-norm be $\eta(\Theta)\coloneq \max_{y,z\in\Theta}\|y-z\|_1$. Define the following \textbf{type intensity}
\begin{align}
\label{eq:kappa}
    % \kappa(\Theta)\coloneq \max_{y,z\in\Theta}\left(\sum_{i:y_i=0}z_i-\sum_{i:y_i>0}z_i\right).
    \kappa(\Theta)\coloneq \max_{y,z\in\Theta}\left(\sum_{i:y_i\leq z_i}z_i-\sum_{i:y_i>z_i}z_i\right) \text{ subject to } \sum_{i:y_i\leq z_i}y_i< \sum_{i:y_i>z_i}y_i.
\end{align}
Furthermore, with fixed $\Theta$ and $A$, we define the maximum and value of the game by
\begin{align}
\label{eq:mu_nu}
    \mu_{\Theta}(A) \coloneq \max_{y\in\Theta}\max_{x\in\mathsf{P}_a}\left|x^{\top}A y\right| \ \ \textbf{(maximum)}, \quad 
    \nu_{\Theta}(A) \coloneq \min_{y\in\Theta}\max_{x\in\mathsf{P}_a}x^{\top}A y \ \ \textbf{(value)}.
\end{align}
\end{definition}

The type intensity, as defined in~\eqref{eq:kappa}, quantifies the divergence between two distributions within the set $\Theta$, specifically those that maximize the given objective. As the density of $\Theta$ increases, the value of $\kappa(\Theta)$ approaches $1$.
% Our first result states upper bound on opportunity and risk for normal-form games.

\begin{theorem}[\textsc{NFG Existence}]
\label{thm:normal_form_upper}
Fix any $\Theta$ and
consider a general-sum normal-form game where Player 1 has a payoff matrix $A\in\mathbb{R}^{a\times b}$ with $\mu_{\Theta}(A)\leq \mu$ and $\nu_{\Theta}(A)\geq \nu$. For any $0\leq \lambda\leq 1$, there exists a mixed strategy $\pi:\mathsf{P}_{\Theta}\rightarrow\mathsf{P}_{a}$ for Player 1 misses $(1-\lambda)\left(\mu-\nu\right)$ opportunity and has $\left((1-\lambda) \left(\mu-\nu\right) + \lambda\mu\eta(\Theta)\right)$ risk. 
\end{theorem}

To show Theorem~\ref{thm:normal_form_upper}, we construct a mixed strategy as follows.
Denote by $\overline{x}$ as the following safe/minimax strategy of Player 1, which is the Nash equilibrium when the game is zero-sum:
\begin{align}
\label{eq:safety}
 \min_{y\in\Theta} \overline{x}^{\top} A y =   \max_{x\in\mathsf{P}_a} \min_{y\in\Theta} x^{\top} A y.
\end{align}
Motivated by the matching pennies example in Section~\ref{sec:matching_pennies}, Player 1 implements a mixed strategy $\pi(\rho) \coloneq \lambda \widetilde{x} + (1-\lambda)\overline{x}$ given the predicted belief $\rho$ as a distribution over types in $\Theta$, where $\widetilde{x}\in \mathsf{P}_a$ is a best response strategy given $\rho$ such that ${x}^{\top} A \mathbb{E}_{\rho}[y]$ is maximized. The detailed proof of Theorem~\ref{thm:normal_form_upper} is provided in Appendix~\ref{app:proof_normal_form_upper}.

Moreover, we further show the following impossibility result, indicating that the tradeoff in Theorem~\ref{thm:normal_form_upper} is tight. We relegate the proof of Theorem~\ref{thm:normal_form_lower} to Appendix~\ref{app:proof_normal_form_lower}. 
% Together, the upper and lower bounds are illustrated in Figure~\ref{fig:normal_form}.

\begin{theorem}[\textsc{NFG Impossibility}]
\label{thm:normal_form_lower}
For any $\Theta$ satisfying $\kappa(\Theta)\geq 0$, there is a payoff matrix $A\in\mathbb{R}^{a\times b}$ with $\mu_{\Theta}(A)\leq \mu$ and  $\nu_{\Theta}(A)\geq \nu$ such that for any $0\leq \lambda\leq 1$, if any mixed strategy $\pi:\mathsf{P}_{\Theta}\rightarrow\mathsf{P}_{a}$ for Player 1 misses at most $(1-\lambda)\left(\mu-\nu\right)$ opportunity, then it incurs at least $\left(\kappa(\Theta)\mu-\nu\right)\left(1+\lambda\right)$ risk. 
\end{theorem}

% \begin{figure}[ht]
% \centering
% \includegraphics[width=0.65\textwidth]{normal_form_type.png}
% \caption{Comparison of upper and lower bounds in Theorem~\ref{thm:normal_form_upper} and~\ref{thm:normal_form_lower} with varying value of the game $\nu$ and type intensity $\kappa(\Theta)$.}
% \label{fig:normal_form}
% \end{figure}

In particular, Theorem~\ref{thm:normal_form_upper} and~\ref{thm:normal_form_lower} together imply the following special case when the hypothesis set $\Theta$ contains all possible mixed strategies in $\mathsf{P}_b$. This corollary highlights that for a fair game, the tradeoff is tight and the mixed strategy $\pi(\rho) \coloneq \lambda \widetilde{x} + (1-\lambda)\overline{x}$ with $\lambda\in [0,1]$ is Pareto optimal when the hypothesis set contains all possible mixed strategies.

\begin{corollary}[\textsc{NFG Pareto Optimality}]
\label{coro:nfg}
Suppose $\Theta=\mathsf{P}_b$ and the game is fair, where Player 1 has a payoff matrix $A\in\mathbb{R}^{a\times b}$ with $\|A\|_{\max}\leq \alpha$. For any $0\leq \lambda\leq 1$, there exists a mixed strategy $\pi$ for Player 1 that misses $(1-\lambda)\alpha$ opportunity and has $(1+\lambda)\alpha$ risk. Furthermore, there is a payoff matrix $A\in\mathbb{R}^{a\times b}$ with $\|A\|_{\max}\leq \alpha$ such that for any $0\leq \lambda\leq 1$ if any mixed strategy $\pi$ for Player 1 misses $\left(1-\lambda\right)\alpha$ opportunity, then it incurs at least $\left(1+\lambda\right)\alpha$ risk. 
\end{corollary}

It is worth noting that the requirement can be relaxed straightforwardly so that as long as the hypothesis set contains two mixed strategies such that $\kappa(\Theta)=1$ in~\eqref{eq:kappa}, the statement holds. The proof of Corollary~\ref{coro:nfg} can be found in Appendix~\ref{app:proof_normal_form_coro}.

\section{Stochastic Bayesian Games  with Untrusted Type Beliefs}
\label{sec:sbg}

In many real-world applications, agents are coupled through a shared state in stochastic Bayesian games~\cite{albrecht2015game,albrecht2016belief}, which combines standard Bayesian games~\cite{harsanyi1967games,harsanyi1968games} with the stochastic games~\cite{shapley1953stochastic}. 
In this section, we consider an infinite-horizon time-varying discounted $\nn$-player 
stochastic Bayesian game, represented by a tuple $\mathcal{M}\coloneq \langle \mathcal{S}, \mathcal{A}, \Theta, \sigma, p, r, \gamma\rangle$, and analyze the impact of beliefs on the opportunity-risk tradeoff.

\subsection{Preliminaries: MDP for Stochastic Bayesian Games}
Let $\mathcal{S}$ be a finite set of states. Write $\mathcal{N}\coloneq \{1,\ldots,\nn\}$.
Let $\mathcal{A}(i)$ be the set of finitely many actions that player $i\in\mathcal{N}$ can take at any state $s \in \mathcal{S}.$ Furthermore, $\mathcal{A}:=\mathcal{A}(1)\times\cdots\times\mathcal{A}(\nn)$ denotes the set of action profiles $a=\left(a(i):i\in\mathcal{N}\right)$ with $a(i) \in \mathcal{A}(i)$. The types of other players are unknown to $i$-th player. The set $\smash{\Theta=\prod_{j\neq i}\Theta(j)}$ is a product type space where each opponent $j$ chooses types from $\Theta_j$.
We focus on the decision-making of the $i$-th player, considering stationary (Markov) strategies.\footnote{Note that if all the other players use Markov strategies, then the considered player has a best response that is a Markov strategy. } At each time $t\geq 0$ each player $j\neq i$ uses a mixed strategy $\pi_j:\mathcal{S}\times\Theta(j)\rightarrow \mathsf{P}_{\mathcal{A}(j)}$ parameterized by a type that depends on the current state and the true type $\theta_j^{\star}\in\Theta(j)$.
We use $r: \mathcal{S}\times \mathcal{A} \rightarrow \mathbb{R}$ to denote the reward function for the $i$-th player, which is assumed to be bounded, as formally defined below.
\begin{assumption}
\label{assumption:reward_bound}
The reward $r:\mathcal{S}\times\mathcal{A}\rightarrow \mathbb{R}$ satisfies that $|r(s,a)|\leq r_{\max}$ for all $s\in\mathcal{S}$, $a\in \mathcal{A}$.
\end{assumption}
For any pair of states $(s, s')$ and action profile $a \in \mathcal{A}$, we define $p(s'|s,a)$ as a transition probability from $s$ to $s'$ given an action profile $a$. Finally, let $\gamma\in(0,1)$ be a discount factor.
% For the first player, we denote by $\pi\left(s\right)\in \mathsf{P}_{\mathcal{A}}$ the mixed strategy (distribution) for player $1$ at state $s$ and $\sigma\left(s\right)\in \mathcal{A}) \in\Delta(\mathcal{A})$ the mixed strategy for player $2$. Denote by $\pi\in \Pi$ and $\sigma\in \Pi$
% the strategies of the two players.  Correspondingly, $a_t=\left(a_t(1), a_t(2)\right)$ is the action profile at time $t\geq 0$.
We define the expected utility of the $i$-th player using a strategy $\pi_i$ as the expected discounted total payoff
\begin{align}
\label{eq:payoff}
J\left(\pi_i;\theta_{-i}^{\star}\right):=\mathbb{E}\left [\sum_{t=0}^{\infty} \gamma^t r\left(s_t, a_t\right)\right ],
\end{align}
where $\theta_{-i}^{\star}\coloneq (\theta_j^{\star}\in\Theta(j):j\neq i)$,
with respect to stochastic processes $\left (s_t \sim p\left(\cdot \mid s_{t-1}, a_{t-1}\right)\right)_{t>0}$ and $\left(a_t(i) \sim \pi_i\left(s_t\right)\right)_{t \geq 0}$, which generate the state and the action trajectories. The expectation in~\eqref{eq:payoff} is taken with respect to all randomness induced by the initial state distribution $s_0 \sim p_0 \in \mathsf{P}_{\mathcal{S}}$, the state transition kernel $p$, and strategy profile $\pi$ (as well as opponents' types in $\Theta$).

\subsection{Opportunity, Risk, and Type Beliefs}

Suppose the $i$-th player has beliefs of the other players' types, provided by a machine-learned forecaster, denoted by $\theta_{-i}$. For notational simplicity, we write $\pi_i$, the strategy for the $i$-th player as $\pi$, and the strategies $\pi_{-i}$ for the opponents as $\sigma$ (parameterized by $\theta_{-i}$). Furthermore, we omit the subscripts and denote the type beliefs and true types  as $\theta$ and $\theta^{\star}$, and write the payoff in~\eqref{eq:payoff} as $J\left(\pi\right)$ if there is no ambiguity.
The $i$-th player uses a strategy $\pi:\mathcal{S}\times\Theta\rightarrow \mathsf{P}_{\mathcal{A}(i)}$, which is a function of the type beliefs of the other players. 
Given a strategy $\pi$ used by the agent and strategies by other players, denoted by $\sigma$, 
we  define the following useful value functions of the game:
\begin{align}
\label{eq:value_function_pi}
V^{\pi,\sigma}(s) \coloneq\mathbb{E}_{p,\pi,\sigma}\left [\sum_{\tau=t}^{\infty} \gamma^{\tau-t} r\left(s_t, a_t\right) | s_t=s \right],
\end{align}
which satisfies the Bellman equation
$
V^{\pi,\sigma}(s) = \mathbb{E}_{a\sim\pi(s),a'\sim\sigma(s)} \left[\left(r + \gamma\mathbb{P} V^{\pi,\sigma}\right)(s,(a,a'))\right],
$
where we define an operator $\mathbb{P} V^{\pi}(s,(a,a'))\coloneq \mathbb{E}_{s'\sim p(\cdot|s,(a,a'))}\left[V^{\pi,\sigma}(s')\right]$. 

The following worst-case payoff gap is defined similarly as the one in~\eqref{eq:gap_normal_form} for normal-form games, which in our context can be considered as the dynamic regret for a strategy $\pi_i$ against an optimal strategy. Note that an optimal strategy is also a best response strategy $\pi^{\star}$ knowing the true types of all players in hindsight maximizing~\eqref{eq:value_function_pi}. The optimal value function given $\sigma$ and $\pi^{\star}$ is denoted by $V^{\star,\sigma}(s)$, whose Bellman optimality equation is
$
V^{\star,\sigma}(s) = \max_{a\in\mathcal{A}} \left(r + \gamma\mathbb{P}^{\sigma} V^{\star,\sigma}\right)(s,a).
$

\begin{definition}
\label{def:payoff_sbg}
% Fix some metric $d:\Theta'\times\Theta'\rightarrow \mathbb{R}_+$ between types $\theta_{-i}$ and $\theta_{-i}^{\star}$ where $\Theta'\coloneq \prod_{j\neq i}\Theta_{j}$.
Given a stochastic Bayesian game $\mathcal{M}= \langle \mathcal{S}, \mathcal{A}, \Theta, \sigma, p, r, \gamma\rangle$,
for a fixed strategy $\pi$, the payoff gap is defined as 
\begin{align}
\label{eq:epsilon_dynamic_regret}
\Delta_{\mathsf{SBG}}(\varepsilon;\pi)\coloneq \sup_{d\left(\theta,\theta^{\star}\right)\leq \varepsilon}\sup_{s_0\in\mathcal{S}} \Big(V^{\pi^{\star},\sigma(\theta^{\star})}(s_0) - V^{\pi(\theta),\sigma(\theta^{\star})}(s_0)\Big),
\end{align}
where $s_0$ denotes an initial state and $d\left(\theta,\theta^{\star}\right)\coloneq\max_{s\in\mathcal{S}}\left\|\sigma(s;\theta)-\sigma(s;\theta^{\star})\right\|_1\leq \varepsilon$.
\end{definition}

Similar to normal-form Bayesian games, we define the (missed) opportunity and risk below.
\begin{definition}
Given a payoff gap $\Delta_{\mathsf{SBG}}(\varepsilon;\pi)$ in~\eqref{eq:epsilon_dynamic_regret}, the \textit{(missed) opportunity} of $\pi$ is 
$\beta(\pi)\coloneq\Delta_{\mathsf{SBG}}(0;\pi)$ and the \textit{risk} is $\delta(\pi)\coloneq\max_{\varepsilon>0}\Delta_{\mathsf{SBG}}(\varepsilon;\pi)$
% Moreover, a strategy $\pi$ is $\beta$-foregone if  $\Delta_{\mathsf{SBG}}(0;\pi)\leq \beta(\pi)$ and $\delta$-risky if $\max_{\varepsilon>0}\Delta_{\mathsf{SBG}}(\varepsilon;\pi)\leq \delta(\pi)$.
\end{definition}

Our goal is to characterize a tradeoff between the opportunity and risk for stochastic Bayesian games. Finally, we define the value of the game.
\begin{definition}
    Given a hypothesis set $\Theta$ and a stochastic Bayesian game $\mathcal{M}= \langle \mathcal{S}, \mathcal{A}, \Theta, \sigma, p, r, \gamma\rangle$, we denote the value of $\mathcal{M}$ by
    \begin{align}
\label{eq:sbg_minimiax}
\nu_{\Theta}(\mathcal{M})\coloneq \max_{\pi\in\Pi} \min_{\theta_{-i}\in\Theta} J(\pi;\theta_{-i}) \ \ \textbf{(value)}.
\end{align}
\end{definition}

\vspace{-7pt}

\subsection{Opportunity-Risk Tradeoff for Stochastic Bayesian Games}

\vspace{-3pt}

% \begin{theorem}
% \label{thm:sbg_upper}
% Consider a stochastic Bayesian game $\mathcal{M}= \langle \mathcal{S}, \mathcal{A}, \Theta, \sigma, p, r, \gamma\rangle$ with $\nu_{\Theta}(\mathcal{M})\geq \nu$. There exists a strategy $\pi$ whose opportunity and risk costs satisfy
% \begin{align*}
%     \beta(\pi)\leq  \frac{1}{1-\lambda\gamma} \left(C_1(\gamma)-\gamma\nu\right)(1-\lambda), \ \text{with } C_1(\gamma)&\coloneq \left(\frac{\gamma^2-3\gamma+6}{(1-\gamma)^2}\right)r_{\max},\\
%     \delta(\pi)\leq  \frac{1}{1-\lambda\gamma}\Big(\left(C_2(\gamma)-\gamma\nu\right)(1-\lambda) + C_3(\gamma)\lambda \Big), \ \text{with } C_2(\gamma)&\coloneq \left(\frac{\gamma^2-3\gamma+2}{1-\gamma}\right)r_{\max}, \\  \text{and } C_3(\gamma)& \coloneq \left(\frac{4-4\gamma^2+2}{(1-\gamma)^2}\right)r_{\max}.
% \end{align*}
% \end{theorem}

\begin{figure}[t]
\centering
\includegraphics[width=0.6\textwidth]{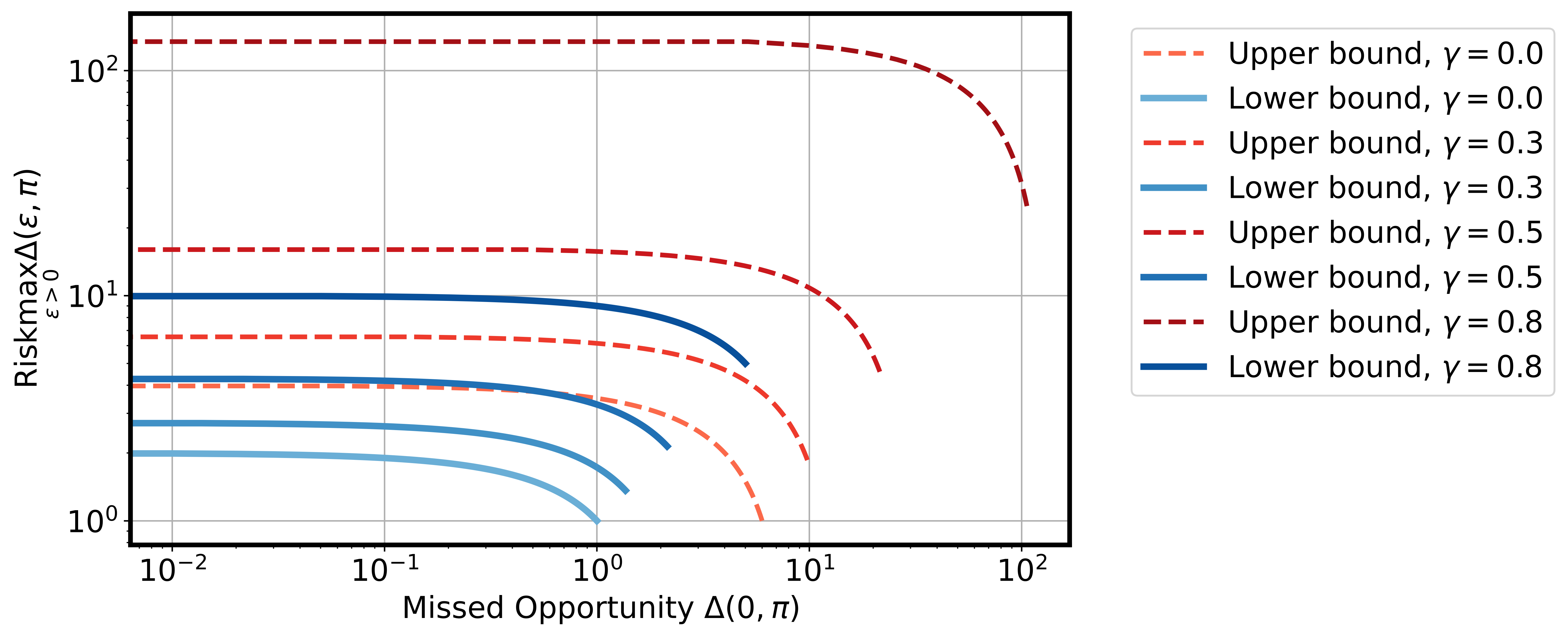}
% \caption{Illustration of the tradeoff in a normal-form game between opportunity and risk with $\alpha=5$ and $V\in \{1,2,\ldots,5\}$.}
\caption{Comparison of lower and upper bounds with a varying discount factor $\gamma$.}
\label{fig:sbg}
\end{figure}

Based on the definitions above, our first result states an upper bound on the opportunity and risk for stochastic Bayesian games. 
Proving the bounds in Theorem~\ref{thm:sbg_upper} is nontrivial because the players are coupled by a shared state. Consequently, the analysis used in Section~\ref{sec:nfg} for a simple convex combination of the best response to type beliefs $\theta$ and a safe strategy in normal-form games cannot be directly applied to the stochastic setting.

\begin{theorem}[\textsc{SBG Existence}]
\label{thm:sbg_upper}
Consider a stochastic Bayesian game $\mathcal{M}= \langle \mathcal{S}, \mathcal{A}, \Theta, \sigma, p, r, \gamma\rangle$ with $\nu_{\Theta}(\mathcal{M})\geq \nu$. For any $0\leq\lambda\leq 1$, there exists a strategy $\pi$ whose (missed) opportunity and risk satisfy
\begin{align*}
    \beta(\pi)\leq &\frac{1}{1-\lambda\gamma} \left(C_1(\gamma)r_{\max}-\gamma\nu\right)(1-\lambda), \ \text{with } C_1(\gamma)\coloneq \frac{\gamma^2-3\gamma+6}{(1-\gamma)^2},\\
    \delta(\pi)\leq & \frac{1}{1-\lambda\gamma}\left(C_2(\gamma)r_{\max}-\gamma\nu\right)(1+\lambda), \ \text{with } C_2(\gamma)\coloneq \max\left(C_3(\gamma),C_4(\gamma)\right), \\ &  \ C_3(\gamma)  \coloneq \frac{\gamma^2-3\gamma+2}{1-\gamma} \text{ and } C_4(\gamma) \coloneq \frac{2-2\gamma^2+1}{(1-\gamma)^2},
\end{align*}
where $r_{\max}$ is defined in Assumption~\ref{assumption:reward_bound}.
\end{theorem}
The proof of Theorem~\ref{thm:sbg_upper} is given in Appendix~\ref{app:proof_sbg_upper}, which constructs the following strategy.

Given type beliefs $\theta$, we denote a parameterized strategy $\smash{\widetilde{\sigma}\coloneq \sigma(\theta)}$, and let $\smash{\overline{\sigma}=\sigma(\overline{\theta})}$ be the safe strategies of the opponents, with $\smash{\overline{\theta}}$ being an optimal solution of the minimax optimization in~\eqref{eq:sbg_minimiax}.
Denoting $\smash{\widetilde{V}\coloneq V^{\star,\widetilde{\sigma}}}$ and $\smash{\overline{V}\coloneq V^{\star,\overline{\sigma}}}$ the optimal value functions with the opponents' strategies being $\smash{\widetilde{\sigma}}$ and $\smash{\overline{\sigma}}$ respectively. 
Unlike the convex combination used to show Theorem~\ref{thm:normal_form_upper} for normal-form Bayesian games,  the strategy constructed for proving Theorem~\ref{thm:sbg_upper} is a value-based strategy below
\begin{align}
\label{eq:sbg_policy}   \pi\left(\theta;s\right) \in \argmax_{a\in\mathsf{P}_{\mathcal{A}(i)}} a^{\top} \left(\lambda R_{\widetilde{V}}(s)\widetilde{\sigma}(s)+(1-\lambda)R_{\overline{V}}(s)\overline{\sigma}(s)\right),
\end{align}
where $\lambda$ is a parameter that indicates the level of trusts on type beliefs, and for a fixed  state $s\in\mathcal{S}$ and a given value function $V:\mathcal{S}\rightarrow\mathbb{R}$, $R_{V}$ is an $\mathcal{A}(i)\times \prod_{j\neq i}\mathcal{A}(j)$ matrix defined as
\begin{align*}
    R_{V}\left(\left(a_i,a_{-i}\right);s\right) \coloneq r\left(s,\left(a_i,a_{-i}\right)\right) + \gamma\sum_{s'\in\mathcal{S}}p\left(s'|s,\left(a_i,a_{-i}\right)\right)V(s).
\end{align*}
As Theorem~\ref{thm:sbg_upper} indicates, this design leads a safe and exploitative play, whose efficacy is validated through case studies in Section~\ref{sec:exp}.
Finally, the following lower bounds can be derived, implying that the upper bounds on $\beta(\pi)$ and $\delta(\pi)$ are tight and the value-based strategy formed in~\eqref{eq:sbg_policy} are Pareto optimal up to 
multiplicative constants. Together, the theoretical bounds are illustrated in Figure~\ref{fig:sbg}.

\begin{theorem}[\textsc{SBG Impossibility}]
\label{thm:sbg_lower}
 There is a stochastic Bayesian game $\mathcal{M}= \langle \mathcal{S}, \mathcal{A}, \Theta, \sigma, p, r, \gamma\rangle$ whose value satisfies $\nu_{\Theta}(\mathcal{M})\geq \nu$ such that for any $0\leq\lambda\leq 1$, if any strategy $\pi$ misses at most $\frac{1}{1-\lambda\gamma}\left(r_{\max}-\nu\right)(1-\lambda)$ opportunity, then it incurs at least $\frac{1}{1-\lambda\gamma}\left(r_{\max}-\nu\right)(1+\lambda)$ risk.
\end{theorem}
In Theorem~\ref{thm:sbg_lower}, the bounds on (missed) opportunity and risk differ from those in Theorem~\ref{thm:sbg_upper} by multiplicative constants.
Theorem~\ref{thm:sbg_lower} can be derived by applying Theorem~\ref{thm:normal_form_lower} for normal form games to a stateless MDP.
We relegate the proof of Theorem~\ref{thm:sbg_lower} to Appendix~\ref{app:proof_sbg_lower}.

\section{Case Studies}
\label{sec:exp}

\begin{figure}[t]
\centering
\includegraphics[width=0.7\textwidth]{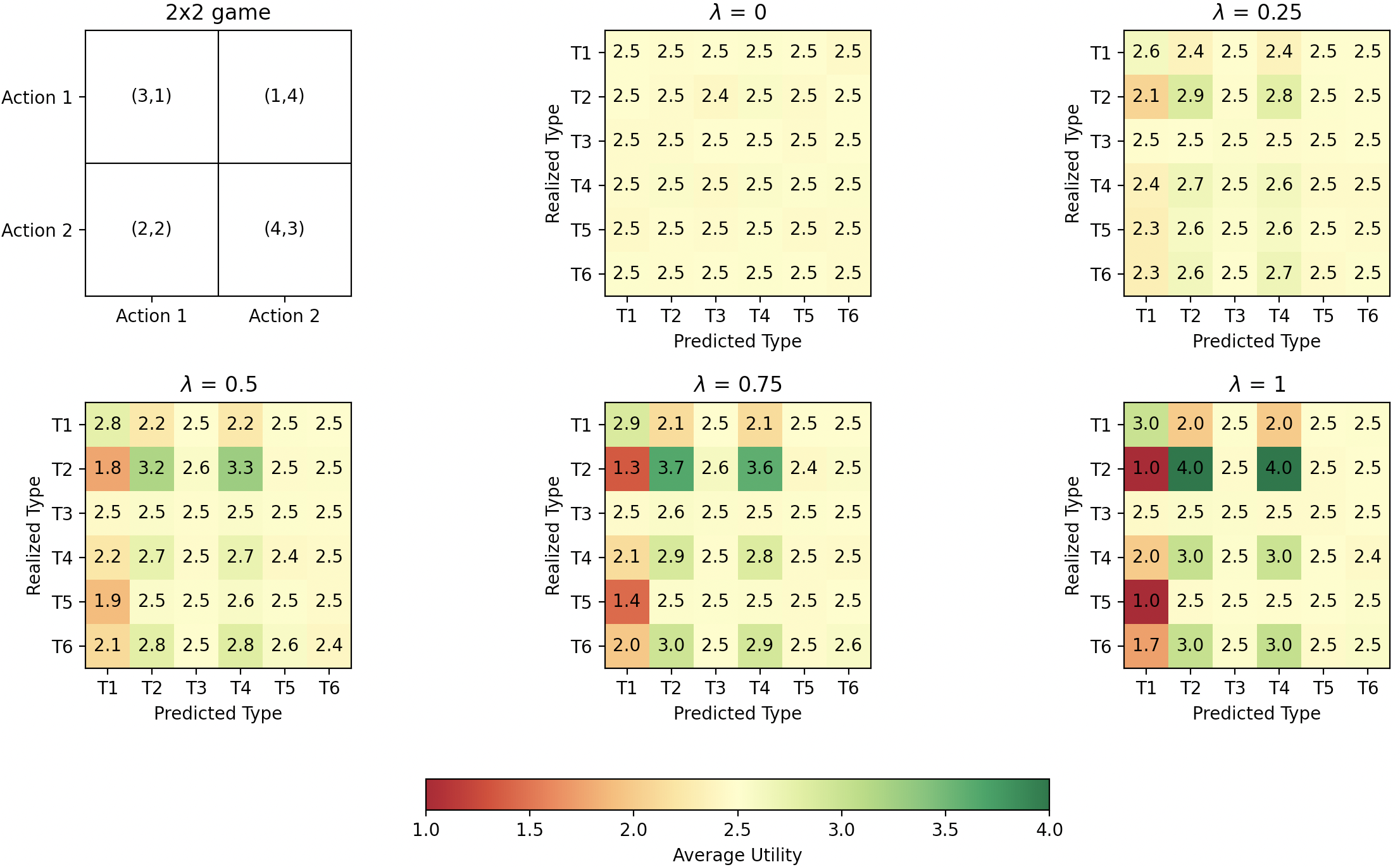}
% \caption{Illustration of the tradeoff in a normal-form game between opportunity and risk with $\alpha=5$ and $V\in \{1,2,\ldots,5\}$.}
\caption{Comparison of average payoff for a player when varying values of $\lambda$ and 6 potential discrete types for an instantiation of a $2 \times 2$ game.}
\label{fig:2x2_experiment}
\end{figure}

\textbf{$2\times 2$ Game.}
We provide a set of numerical evaluations to showcase the opportunity-risk tradeoff across a range of game theoretic scenarios. We begin with a focus on $2 \times 2$ games and to that end we sample games spanning the topology of $2 \times 2$ matrix games as shown in \cite{Bruns_2015}. This topology includes the comprehensive set of 78 strictly ordinal $2 \times 2$ games introduced by \cite{rapoportguyer1966}. We formulate our evaluation as a single state Stochastic Bayesian Game with the state being a specific $2 \times 2$ game. This evaluation is helpful in providing a concrete illustration of the tradeoff dynamics on a wide range of games within a particular topology given the comprehensive bench-marking done on the set of $2 \times 2$ games. 

We make use of 3 broad type classes (Markovian, Leader-Follower-Trigger-Agents, and Co-evoloved Neural Networks) inspired from work by \cite{albrecht2015belief}. It is from these classes that we construct our types which are then used in simulations with a player leveraging a strategy that is a convex combination of a safe strategy and the best response give type beliefs or predictions. Figs \ref{fig:2x2_experiment} and \ref{fig:tradeoff} shows the tradeoff in one particular game as an agent moves from being fully robust as indicated by $\lambda = 0$ to fully trusting of the advice when $\lambda = 1$. In Appendix \ref{app:experiments}, we provide further information on the particular types as well as illustrations on a couple of other games sampled from the topology of $2 \times 2$ games to showcase the range of tradeoffs which exist in this landscape. 

\textbf{Security Game.} Extending our empirical study to real world settings, we also provide an evaluation of the opportunity-risk tradeoff using data from wildlife tracking studies. Security games are a clear domain wherein the tradeoff between opportunity and risk is very critical to the application. Building on a large body of work that has sought to understand the implications of game theoretic analysis in the environmental conservation domain, we formulate and explore the opportunity-risk tradeoff using data from real world wildlife movements. In particular, we formulate a green security game motivated by works such as \cite{fang2015when} wherein we simulate a defender who is trying to protect an elephant population from attackers who are illegal poachers. The defender and attacker engage a multi-state Stochastic Bayesian Game where the states are constructed from historic elephant movement data sourced from \cite{Movebank}. Our evaluation shows the practicality and wide ranging impact of this investigation. Figure \ref{fig:elephant} shows the tradeoff for this setting evaluated for each particular type. We include implementation details on this game in Appendix \ref{app:experiments}.

\begin{figure}[t]
\centering
\includegraphics[width=0.7\textwidth]{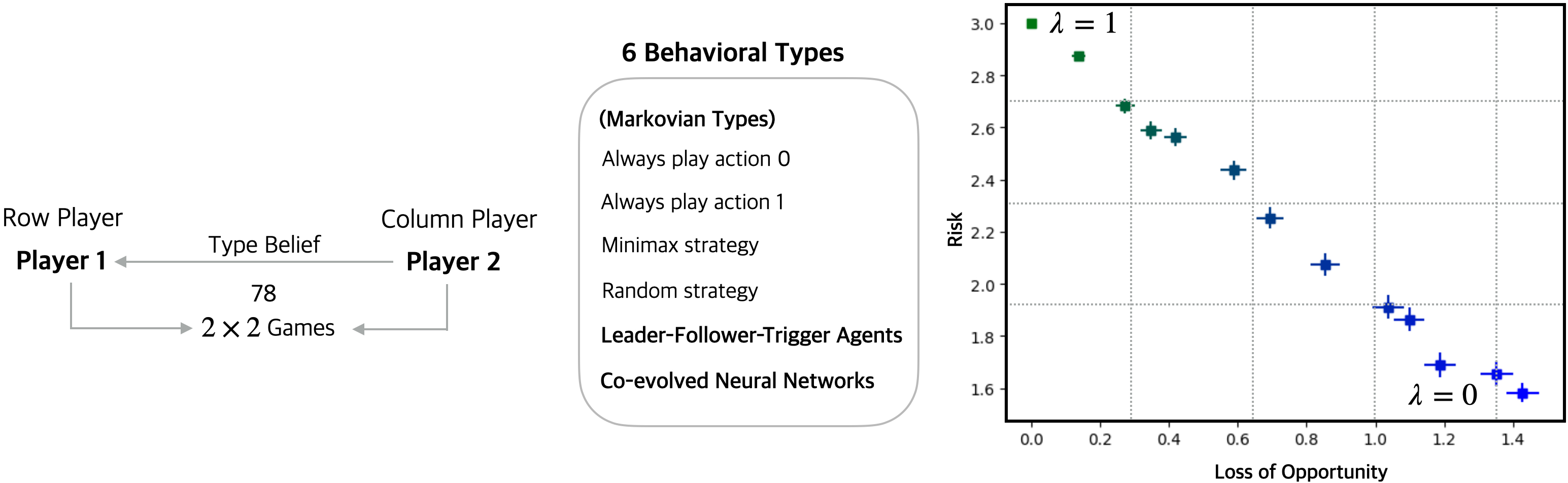}
% \caption{Illustration of the tradeoff in a normal-form game between opportunity and risk with $\alpha=5$ and $V\in \{1,2,\ldots,5\}$.}
\caption{\textbf{Left}: $2\times 2$ games considered in our case study; \textbf{Right}: Opportunity-risk tradeoff in the evaluation of a $2 \times 2$ game using an algorithm that has varying trust of type beliefs in 1,000 random runs. Fully trusting  ($\lambda=1$) and distrusting ($\lambda=0$) type beliefs yield a best response strategy and a minimax strategy correspondingly.}
\label{fig:tradeoff}
\end{figure}
\begin{figure}[t]
\centering
\includegraphics[width=0.7\textwidth]{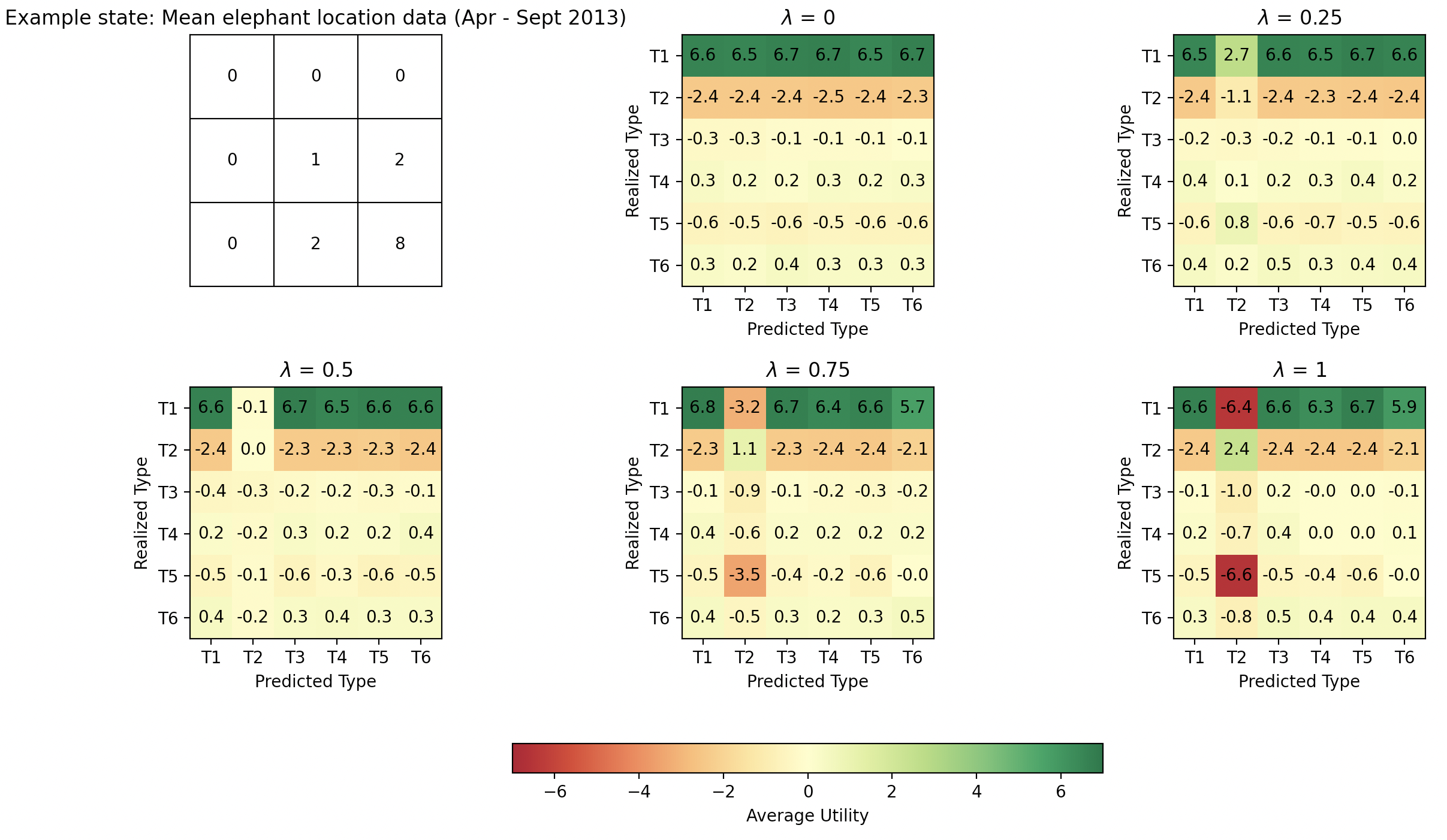}
% \caption{Illustration of the tradeoff in a normal-form game between opportunity and risk with $\alpha=5$ and $V\in \{1,2,\ldots,5\}$.}
\caption{Comparison of average payoff for the defender in a security game protecting an elephant population against illegal poachers when varying values of $\lambda$ and 6 potential discrete attacker types.}
\label{fig:elephant}
\end{figure}

% Three behavior classes
% are Leader-Follower-Trigger Agents (LFT), Co-Evolved Decision Trees (CDT), and Co-Evolved
% Neural Networks (CNN).

% Prior beliefs.

\vspace{-6pt}

\section{Concluding Remarks}

\vspace{-3pt}

In this study, we explored the fundamental limits of safe and exploitative strategies within both normal-form and stochastic Bayesian games, where pre-established type beliefs about opponents are considered. Given that these type beliefs may be inaccurate, relying on them to exploit opponents can result in high-risk strategies. Conversely, not leveraging these beliefs yields overly cautious play, leading to missed opportunities. We have quantified these dynamics by providing upper and lower bounds on the payoff gaps corresponding to different type belief inaccuracies, thereby characterizing the tradeoffs between opportunity and risk. These bounds are consistent up to multiplicative constants. 

\textbf{Limitations and Future Directions.} In our current problem setting, the dynamics of a stochastic Bayesian game is assumed to be stationary, aligning with the canonical models in the related literature (for example~\cite{albrecht2015belief}). To address this limitation,
we plan to extend our framework to include time-varying type beliefs, addressing the absence of analysis for MDPs with dynamic transition probabilities and reward structures. Additionally, refining the opportunity and risk bounds to make them tighter would provide more precise strategic insights.

\newpage

\addcontentsline{toc}{section}{Bibliography}
\bibliographystyle{plainnat}
{\bibliography{main}}

\begin{thebibliography}{56}
\providecommand{\natexlab}[1]{#1}
\providecommand{\url}[1]{\texttt{#1}}
\expandafter\ifx\csname urlstyle\endcsname\relax
  \providecommand{\doi}[1]{doi: #1}\else
  \providecommand{\doi}{doi: \begingroup \urlstyle{rm}\Url}\fi

\bibitem[Albrecht and Ramamoorthy(2015)]{albrecht2015game}
Stefano~V Albrecht and Subramanian Ramamoorthy.
\newblock A game-theoretic model and best-response learning method for ad hoc coordination in multiagent systems.
\newblock \emph{arXiv preprint arXiv:1506.01170}, 2015.

\bibitem[Albrecht and Stone(2018)]{albrecht2018autonomous}
Stefano~V Albrecht and Peter Stone.
\newblock Autonomous agents modelling other agents: A comprehensive survey and open problems.
\newblock \emph{Artificial Intelligence}, 258:\penalty0 66--95, 2018.

\bibitem[Albrecht et~al.(2015)Albrecht, Crandall, and Ramamoorthyc]{albrecht2015belief}
Stefano~V. Albrecht, Jacob~W. Crandall, and Subramanian Ramamoorthyc.
\newblock Belief and truth in hypothesised behaviours.
\newblock \emph{arXiv preprint arXiv:1507.07688}, 2015.

\bibitem[Albrecht et~al.(2016)Albrecht, Crandall, and Ramamoorthy]{albrecht2016belief}
Stefano~V Albrecht, Jacob~W Crandall, and Subramanian Ramamoorthy.
\newblock Belief and truth in hypothesised behaviours.
\newblock \emph{Artificial Intelligence}, 235:\penalty0 63--94, 2016.

\bibitem[Amir et~al.(2022)Amir, Azov, Koren, and Livni]{amir2022better}
Idan Amir, Guy Azov, Tomer Koren, and Roi Livni.
\newblock Better best of both worlds bounds for bandits with switching costs.
\newblock \emph{Advances in neural information processing systems}, 35:\penalty0 15800--15810, 2022.

\bibitem[Anagnostides et~al.(2024)Anagnostides, Panageas, Farina, and Sandholm]{anagnostides2024convergence}
Ioannis Anagnostides, Ioannis Panageas, Gabriele Farina, and Tuomas Sandholm.
\newblock On the convergence of no-regret learning dynamics in time-varying games.
\newblock \emph{Advances in Neural Information Processing Systems}, 36, 2024.

\bibitem[Antoniadis et~al.(2020)Antoniadis, Coester, Elias, Polak, and Simon]{antoniadis2020online}
Antonios Antoniadis, Christian Coester, Marek Elias, Adam Polak, and Bertrand Simon.
\newblock Online metric algorithms with untrusted predictions.
\newblock In \emph{International Conference on Machine Learning}, pages 345--355. PMLR, 2020.

\bibitem[Bichler et~al.(2021)Bichler, Fichtl, Heidekr{\"u}ger, Kohring, and Sutterer]{bichler2021learning}
Martin Bichler, Maximilian Fichtl, Stefan Heidekr{\"u}ger, Nils Kohring, and Paul Sutterer.
\newblock Learning equilibria in symmetric auction games using artificial neural networks.
\newblock \emph{Nature machine intelligence}, 3\penalty0 (8):\penalty0 687--695, 2021.

\bibitem[Bichler et~al.(2023)Bichler, Fichtl, and Oberlechner]{bichler2023computing}
Martin Bichler, Max Fichtl, and Matthias Oberlechner.
\newblock Computing bayes--nash equilibrium strategies in auction games via simultaneous online dual averaging.
\newblock \emph{Operations Research}, 2023.

\bibitem[Brown and Sandholm(2019)]{brown2019superhuman}
Noam Brown and Tuomas Sandholm.
\newblock Superhuman ai for multiplayer poker.
\newblock \emph{Science}, 365\penalty0 (6456):\penalty0 885--890, 2019.

\bibitem[Bruns(2015)]{Bruns_2015}
Bryan Bruns.
\newblock Names for games: Locating 2 × 2 games.
\newblock \emph{Games}, 6\penalty0 (4):\penalty0 495–520, October 2015.
\newblock ISSN 2073-4336.
\newblock \doi{10.3390/g6040495}.
\newblock URL \url{http://dx.doi.org/10.3390/g6040495}.

\bibitem[Chamaille-Jammes(2009-2017)]{Movebank}
Chamaille-Jammes.
\newblock African elephant (migration) chamaillé-jammes hwange np, 2009-2017.
\newblock URL \url{https://www.movebank.org/cms/webapp?gwt_fragment=page=studies,path=study307786785}.

\bibitem[Christianson et~al.(2022{\natexlab{a}})Christianson, Handina, and Wierman]{christianson2022chasing}
Nicolas Christianson, Tinashe Handina, and Adam Wierman.
\newblock Chasing convex bodies and functions with black-box advice.
\newblock In \emph{Conference on Learning Theory}, pages 867--908. PMLR, 2022{\natexlab{a}}.

\bibitem[Christianson et~al.(2022{\natexlab{b}})Christianson, Yeh, Li, Rad, Golmohammadi, and Wierman]{christianson2022robustifying}
Nicolas Christianson, Christopher Yeh, Tongxin Li, Mahdi~Torabi Rad, Azarang Golmohammadi, and Adam Wierman.
\newblock Robustifying machine-learned algorithms for efficient grid operation.
\newblock In \emph{NeurIPS 2022 Workshop on Tackling Climate Change with Machine Learning}, 2022{\natexlab{b}}.

\bibitem[Crandall(2014)]{crandall2014towards}
Jacob~W. Crandall.
\newblock Towards minimizing disappointment in repeated games.
\newblock \emph{Journal of Artificial Intelligence Research (JAIR)}, 49:\penalty0 111--142, 2014.

\bibitem[Dekel et~al.(2004)Dekel, Fudenberg, and Levine]{dekel2004learning}
Eddie Dekel, Drew Fudenberg, and David~K Levine.
\newblock Learning to play bayesian games.
\newblock \emph{Games and economic behavior}, 46\penalty0 (2):\penalty0 282--303, 2004.

\bibitem[Fang et~al.(2015)Fang, Stone, and Tambe]{fang2015when}
Fei Fang, Peter Stone, and Milind Tambe.
\newblock When security games go green: Designing defender strategies to prevent poaching and illegal fishing.
\newblock \emph{IJCAI}, 2015.

\bibitem[Foster and Young(2001)]{foster2001impossibility}
Dean~P Foster and H~Peyton Young.
\newblock On the impossibility of predicting the behavior of rational agents.
\newblock \emph{Proceedings of the National Academy of Sciences}, 98\penalty0 (22):\penalty0 12848--12853, 2001.

\bibitem[Friedman(1971)]{friedman1971optimal}
Lawrence Friedman.
\newblock Optimal bluffing strategies in poker.
\newblock \emph{Management Science}, 17\penalty0 (12):\penalty0 B--764, 1971.

\bibitem[Golowich and Moitra(2022)]{golowich2022can}
Noah Golowich and Ankur Moitra.
\newblock Can q-learning be improved with advice?
\newblock In \emph{Conference on Learning Theory}, pages 4548--4619. PMLR, 2022.

\bibitem[Harsanyi(1967)]{harsanyi1967games}
John~C Harsanyi.
\newblock Games with incomplete information played by “bayesian” players, i--iii part i. the basic model.
\newblock \emph{Management science}, 14\penalty0 (3):\penalty0 159--182, 1967.

\bibitem[Harsanyi(1968)]{harsanyi1968games}
John~C Harsanyi.
\newblock Games with incomplete information played by “bayesian” players part ii. bayesian equilibrium points.
\newblock \emph{Management science}, 14\penalty0 (5):\penalty0 320--334, 1968.

\bibitem[Im et~al.(2022)Im, Kumar, Petety, and Purohit]{im2022parsimonious}
Sungjin Im, Ravi Kumar, Aditya Petety, and Manish Purohit.
\newblock Parsimonious learning-augmented caching.
\newblock In \emph{International Conference on Machine Learning}, pages 9588--9601. PMLR, 2022.

\bibitem[Jin et~al.(2021)Jin, Huang, and Luo]{jin2021best}
Tiancheng Jin, Longbo Huang, and Haipeng Luo.
\newblock The best of both worlds: stochastic and adversarial episodic mdps with unknown transition.
\newblock \emph{Advances in Neural Information Processing Systems}, 34:\penalty0 20491--20502, 2021.

\bibitem[Jordan(1991)]{jordan1991bayesian}
James~S Jordan.
\newblock Bayesian learning in normal form games.
\newblock \emph{Games and Economic Behavior}, 3\penalty0 (1):\penalty0 60--81, 1991.

\bibitem[Kalai and Lehrer(1993)]{kalai1993rational}
Ehud Kalai and Ehud Lehrer.
\newblock Rational learning leads to nash equilibrium.
\newblock \emph{Econometrica: Journal of the Econometric Society}, pages 1019--1045, 1993.

\bibitem[Koza(1992)]{koza1992genetic}
John~R. Koza.
\newblock \emph{Genetic Programming: On the Programming of Computers by Means of Natural Selection}.
\newblock The MIT Press, 1992.

\bibitem[Li et~al.(2024{\natexlab{a}})Li, Yang, Wierman, and Ren]{li2024robust}
Pengfei Li, Jianyi Yang, Adam Wierman, and Shaolei Ren.
\newblock Robust learning for smoothed online convex optimization with feedback delay.
\newblock \emph{Advances in Neural Information Processing Systems}, 36, 2024{\natexlab{a}}.

\bibitem[Li et~al.(2023{\natexlab{a}})Li, Guo, Xiu, Xu, Yu, Wang, Liu, Yang, and Liu]{li2023byzantine}
Simin Li, Jun Guo, Jingqiao Xiu, Ruixiao Xu, Xin Yu, Jiakai Wang, Aishan Liu, Yaodong Yang, and Xianglong Liu.
\newblock Byzantine robust cooperative multi-agent reinforcement learning as a bayesian game.
\newblock In \emph{The Twelfth International Conference on Learning Representations}, 2023{\natexlab{a}}.

\bibitem[Li(2023)]{li2023learning}
Tongxin Li.
\newblock \emph{Learning-Augmented Control and Decision-Making: Theory and Applications in Smart Grids}.
\newblock PhD thesis, California Institute of Technology, 2023.

\bibitem[Li and Sun(2024)]{li2024out}
Tongxin Li and Chenxi Sun.
\newblock Out-of-distribution-aware electric vehicle charging.
\newblock \emph{IEEE Transactions on Transportation Electrification}, 2024.

\bibitem[Li et~al.(2021{\natexlab{a}})Li, Chen, Sun, Wierman, and Low]{li2021information}
Tongxin Li, Yue Chen, Bo~Sun, Adam Wierman, and Steven Low.
\newblock Information aggregation for constrained online control.
\newblock \emph{ACM SIGMETRICS Performance Evaluation Review}, 49\penalty0 (1):\penalty0 7--8, 2021{\natexlab{a}}.

\bibitem[Li et~al.(2021{\natexlab{b}})Li, Sun, Chen, Ye, Low, and Wierman]{li2021learning}
Tongxin Li, Bo~Sun, Yue Chen, Zixin Ye, Steven~H Low, and Adam Wierman.
\newblock Learning-based predictive control via real-time aggregate flexibility.
\newblock \emph{IEEE Transactions on Smart Grid}, 12\penalty0 (6):\penalty0 4897--4913, 2021{\natexlab{b}}.

\bibitem[Li et~al.(2022)Li, Yang, Qu, Shi, Yu, Wierman, and Low]{li2022robustness}
Tongxin Li, Ruixiao Yang, Guannan Qu, Guanya Shi, Chenkai Yu, Adam Wierman, and Steven Low.
\newblock Robustness and consistency in linear quadratic control with untrusted predictions.
\newblock \emph{ACM SIGMETRICS Performance Evaluation Review}, 50\penalty0 (1):\penalty0 107--108, 2022.

\bibitem[Li et~al.(2023{\natexlab{b}})Li, Yang, Qu, Lin, Wierman, and Low]{li2023certifying}
Tongxin Li, Ruixiao Yang, Guannan Qu, Yiheng Lin, Adam Wierman, and Steven~H Low.
\newblock Certifying black-box policies with stability for nonlinear control.
\newblock \emph{IEEE Open Journal of Control Systems}, 2:\penalty0 49--62, 2023{\natexlab{b}}.

\bibitem[Li et~al.(2024{\natexlab{b}})Li, Lin, Ren, and Wierman]{li2024beyond}
Tongxin Li, Yiheng Lin, Shaolei Ren, and Adam Wierman.
\newblock Beyond black-box advice: Learning-augmented algorithms for mdps with q-value predictions.
\newblock \emph{Advances in Neural Information Processing Systems}, 36, 2024{\natexlab{b}}.

\bibitem[Lin et~al.(2022)Lin, Hu, Qu, Li, and Wierman]{lin2022bounded}
Yiheng Lin, Yang Hu, Guannan Qu, Tongxin Li, and Adam Wierman.
\newblock Bounded-regret mpc via perturbation analysis: Prediction error, constraints, and nonlinearity.
\newblock \emph{Advances in Neural Information Processing Systems}, 35:\penalty0 36174--36187, 2022.

\bibitem[Lowe et~al.(2017)Lowe, Wu, Tamar, Harb, Pieter~Abbeel, and Mordatch]{lowe2017multi}
Ryan Lowe, Yi~I Wu, Aviv Tamar, Jean Harb, OpenAI Pieter~Abbeel, and Igor Mordatch.
\newblock Multi-agent actor-critic for mixed cooperative-competitive environments.
\newblock \emph{Advances in neural information processing systems}, 30, 2017.

\bibitem[Lykouris and Vassilvitskii(2021)]{lykouris2021competitive}
Thodoris Lykouris and Sergei Vassilvitskii.
\newblock Competitive caching with machine learned advice.
\newblock \emph{Journal of the ACM (JACM)}, 68\penalty0 (4):\penalty0 1--25, 2021.

\bibitem[Mahdian et~al.(2012)Mahdian, Nazerzadeh, and Saberi]{mahdian2012online}
Mohammad Mahdian, Hamid Nazerzadeh, and Amin Saberi.
\newblock Online optimization with uncertain information.
\newblock \emph{ACM Transactions on Algorithms (TALG)}, 8\penalty0 (1):\penalty0 1--29, 2012.

\bibitem[Milec et~al.(2021)Milec, Kub{\'\i}{\v{c}}ek, and Lis{\`y}]{milec2021continual}
David Milec, Ond{\v{r}}ej Kub{\'\i}{\v{c}}ek, and Viliam Lis{\`y}.
\newblock Continual depth-limited responses for computing counter-strategies in sequential games.
\newblock \emph{arXiv preprint arXiv:2112.12594}, 2021.

\bibitem[Milgrom and Roberts(1991)]{milgrom1991adaptive}
Paul Milgrom and John Roberts.
\newblock Adaptive and sophisticated learning in normal form games.
\newblock \emph{Games and economic Behavior}, 3\penalty0 (1):\penalty0 82--100, 1991.

\bibitem[Nachbar(1997)]{nachbar1997prediction}
John~H Nachbar.
\newblock Prediction, optimization, and learning in repeated games.
\newblock \emph{Econometrica: Journal of the Econometric Society}, pages 275--309, 1997.

\bibitem[Nachbar(2005)]{nachbar2005beliefs}
John~H Nachbar.
\newblock Beliefs in repeated games.
\newblock \emph{Econometrica}, 73\penalty0 (2):\penalty0 459--480, 2005.

\bibitem[Purohit et~al.(2018)Purohit, Svitkina, and Kumar]{purohit2018improving}
Manish Purohit, Zoya Svitkina, and Ravi Kumar.
\newblock Improving online algorithms via ml predictions.
\newblock In \emph{Advances in Neural Information Processing Systems}, pages 9661--9670, 2018.

\bibitem[Rapoport and Guyer(1966)]{rapoportguyer1966}
Anatol Rapoport and Melvin Guyer.
\newblock A taxonomy of 2 $\times$ 2 games.
\newblock \emph{General Systems: Yearbook of the Society for General Systems Research}, 11:\penalty0 203--214, 1966.

\bibitem[Rashid et~al.(2020)Rashid, Samvelyan, De~Witt, Farquhar, Foerster, and Whiteson]{rashid2020monotonic}
Tabish Rashid, Mikayel Samvelyan, Christian~Schroeder De~Witt, Gregory Farquhar, Jakob Foerster, and Shimon Whiteson.
\newblock Monotonic value function factorisation for deep multi-agent reinforcement learning.
\newblock \emph{Journal of Machine Learning Research}, 21\penalty0 (178):\penalty0 1--51, 2020.

\bibitem[Rohatgi(2020)]{rohatgi2020near}
Dhruv Rohatgi.
\newblock Near-optimal bounds for online caching with machine learned advice.
\newblock In \emph{Proceedings of the Fourteenth Annual ACM-SIAM Symposium on Discrete Algorithms}, pages 1834--1845. SIAM, 2020.

\bibitem[Shapley(1953)]{shapley1953stochastic}
Lloyd~S Shapley.
\newblock Stochastic games.
\newblock \emph{Proceedings of the national academy of sciences}, 39\penalty0 (10):\penalty0 1095--1100, 1953.

\bibitem[Singh and Yee(1994)]{singh1994upper}
Satinder~P Singh and Richard~C Yee.
\newblock An upper bound on the loss from approximate optimal-value functions.
\newblock \emph{Machine Learning}, 16:\penalty0 227--233, 1994.

\bibitem[Southey et~al.(2012)Southey, Bowling, Larson, Piccione, Burch, Billings, and Rayner]{southey2012bayes}
Finnegan Southey, Michael~P Bowling, Bryce Larson, Carmelo Piccione, Neil Burch, Darse Billings, and Chris Rayner.
\newblock Bayes' bluff: Opponent modelling in poker.
\newblock \emph{arXiv preprint arXiv:1207.1411}, 2012.

\bibitem[Vinyals et~al.(2019)Vinyals, Babuschkin, Czarnecki, Mathieu, Dudzik, Chung, Choi, Powell, Ewalds, Georgiev, et~al.]{vinyals2019grandmaster}
Oriol Vinyals, Igor Babuschkin, Wojciech~M Czarnecki, Micha{\"e}l Mathieu, Andrew Dudzik, Junyoung Chung, David~H Choi, Richard Powell, Timo Ewalds, Petko Georgiev, et~al.
\newblock Grandmaster level in starcraft ii using multi-agent reinforcement learning.
\newblock \emph{Nature}, 575\penalty0 (7782):\penalty0 350--354, 2019.

\bibitem[Wang et~al.(2016)Wang, Wan, Zhang, Li, and Zhang]{wang2016towards}
Shiyong Wang, Jiafu Wan, Daqiang Zhang, Di~Li, and Chunhua Zhang.
\newblock Towards smart factory for industry 4.0: a self-organized multi-agent system with big data based feedback and coordination.
\newblock \emph{Computer networks}, 101:\penalty0 158--168, 2016.

\bibitem[Yan et~al.(2024)Yan, Guo, Lou, Wang, Zhang, and Du]{yan2024efficient}
Xue Yan, Jiaxian Guo, Xingzhou Lou, Jun Wang, Haifeng Zhang, and Yali Du.
\newblock An efficient end-to-end training approach for zero-shot human-ai coordination.
\newblock \emph{Advances in Neural Information Processing Systems}, 36, 2024.

\bibitem[Yang et~al.(2024)Yang, Li, Li, Wierman, and Ren]{yang2024anytime}
Jianyi Yang, Pengfei Li, Tongxin Li, Adam Wierman, and Shaolei Ren.
\newblock Anytime-competitive reinforcement learning with policy prior.
\newblock \emph{Advances in Neural Information Processing Systems}, 36, 2024.

\bibitem[Zadeh(1977)]{zadeh1977computation}
Norman Zadeh.
\newblock Computation of optimal poker strategies.
\newblock \emph{Operations Research}, 25\penalty0 (4):\penalty0 541--562, 1977.

\end{thebibliography}

\newpage

\appendix

\textbf{Broader Impacts.} The implications of our research extend beyond theoretical interests and have practical significance in fields such as economics, cybersecurity, and strategic planning, where decision-making under uncertainty is crucial. By improving the understanding of how predictive information can be used safely and effectively in competitive environments, our work supports the development of more robust strategies in these areas. This can lead to better risk management practices and enhance the ability of systems to make informed decisions even when faced with unreliable or incomplete information. However, there are potential negative impacts, such as the risk of misuse of predictive models that may lead to biased or unfair decisions if the underlying data or assumptions are flawed. Future research could focus on reducing these risks, ensuring AI systems remain reliable and safe.

\section{Proof of Theorem~\ref{thm:normal_form_upper}}
\label{app:proof_normal_form_upper}

We first consider bounding the opportunity, focusing on the case when the predicted belief contains no error, i.e., $y^{\star}=\mathbb{E}_{\rho}[y]$. Then, by the definition of $\pi(\rho)$, for any payoff matrix $A$ with $\mu_{\Theta}(A)\leq \mu$ and $\nu_{\Theta}(A)\geq \nu$, we get
\begin{align}
\nonumber
  \Delta_{\mathsf{NFG}}(0;\pi) =&  \max_{y^{\star}\in \Theta} \left(\max_{x\in \mathsf{P}_a}x^{\top} A y^{\star} - \pi(\rho) A y^{\star}\right) \\
  % \nonumber
  % =& (1-\lambda) \max_{A:\|A\|_{\max}\leq \alpha} \max_{y^{\star}\in \Theta} \left(\max_{x\in \mathsf{P}_a}x^{\top} A y^{\star} - \overline{x}^{\top} A y^{\star}\right)\\
\label{eq:opp_1}
    \leq & (1-\lambda)\left( \max_{x\in\mathsf{P}_a}\max_{y^{\star}\in\Theta} x^{\top} A y^{\star} - \min_{y^{\star}\in \Theta}\overline{x}^{\top} A y^{\star}\right),
\end{align}
where we have maximized the two terms $\max_{x\in \mathsf{P}_a}x^{\top} A y^{\star}$ and $\overline{x}^{\top} A y^{\star}$ separately over $y^{\star}\in \Theta$ to obtain~\eqref{eq:opp_1}. Therefore, noting the definition of the safe strategy $\overline{x}$ in~\eqref{eq:safety}, 
\begin{align}
\label{eq:opp}
  \Delta_{\mathsf{NFG}}(0;\pi) \leq & (1-\lambda)\left( \max_{x\in\mathsf{P}_a}\max_{y^{\star}\in\Theta} x^{\top} A y^{\star} -  \max_{x\in\mathsf{P}_a} \min_{y^{\star}\in\Theta} x^{\top} A y^{\star} \right)
    \leq (1-\lambda)\left(\mu-\nu\right).
    % = &(1-\lambda) \max_{y\in \Delta_n} \left(y^{\top} A (x^*(y)-\overline{x})\right)
\end{align}
% where $V=\max_{x\in\mathsf{P}_a} \min_{y^{\star}\in\Theta} x^{\top} A y^{\star} $ is the value of the general-sum game.

Now, we consider bounding the risk. By the definition of the following mixed strategy $\pi(\rho)$ used by Player 1, for any payoff matrix $A$ with $\mu_{\Theta}(A)\leq \mu$, $\nu_{\Theta}(A)\geq \nu$, and $\|A\|_{\max}\leq \alpha$,
\begin{align}
\nonumber
&\Delta_{\mathsf{NFG}}(\varepsilon;\pi)=\max_{d\left(\rho,y^{\star}\right)\leq \varepsilon} \left(\max_{x\in\mathsf{P}_a}x^{\top}A y^{\star} - \pi(\rho)^{\top}A y^{\star}\right)\\
\label{eq:risk_1}
\leq  &\max_{d\left(\rho,y^{\star}\right)\leq \varepsilon}  \left( \max_{x\in\mathsf{P}_a}x^{\top} A \mathbb{E}_{\rho}[y] - \lambda \widetilde{x}^{\top} A y^{\star}\right) - (1-\lambda)\min_{y^{\star}\in \Theta}\overline{x}^{\top} 
A y^{\star}.
\end{align}
In~\eqref{eq:risk_1}, we replace $\max_{x\in\mathsf{P}_a}x^{\top} A y^{\star}$ by $\max_{x\in\mathsf{P}_a}x^{\top} A \mathbb{E}_{\rho}[y]$ since $y^{\star}\in\Theta$ and there exists a feasible $\rho$ with $\mathbb{E}_{\rho}[y]=y^{\star}$ that always satisfies the constraint $d\left(\rho,y^{\star}\right)\leq \varepsilon$.

Since $\widetilde{x}$ is a best response strategy given $\rho$, $\max_{x\in\mathsf{P}_a}x^{\top} A \mathbb{E}_{\rho}[y] =  {\widetilde{x}}^{\top} A \mathbb{E}_{\rho}[y]$.
Decomposing $\max_{x\in\mathsf{P}_a}x^{\top} A y^{\star}=(1-\lambda)\max_{x\in\mathsf{P}_a}x^{\top} A y^{\star} + \lambda\max_{x\in\mathsf{P}_a}x^{\top} A y^{\star}$ and maximizing the two terms $(1-\lambda)\max_{x\in\mathsf{P}_a}x^{\top} A y^{\star}$ and $\lambda\max_{x\in\mathsf{P}_a}x^{\top} A y^{\star}$ over $d\left(\rho,y^{\star}\right)\leq \varepsilon$ respectively, $\Delta(\varepsilon;\alpha,\pi)$ can be further bounded from above by
\begin{align}
\label{eq:risk_2}
& \max_{A:\|A\|_{\max}\leq \alpha}\left((1-\lambda)\max_{x\in\mathsf{P}_a,y^{\star}\in \Theta} x^{\top} A y^{\star} +\lambda \max_{d\left(\rho,y^{\star}\right)\leq \varepsilon} \left(\widetilde{x}^{\top}A(\mathbb{E}_{\rho}[y]-y^{\star})\right) - (1-\lambda)\nu\right).
\end{align}
Since $d\left(\rho,y^{\star}\right)\coloneq \|\mathbb{E}_{\rho}[y]-y^{\star}\|_1\leq\varepsilon$,~\eqref{eq:risk_2} yields
\begin{align}
\label{eq:risk_3}
\Delta_{\mathsf{NFG}}(\varepsilon;\pi) \leq &
(1-\lambda) \left(\mu-\nu\right) + \lambda\mu\varepsilon.
\end{align}

Maximizing over $\varepsilon\leq \eta(\Theta)$ for~\eqref{eq:risk_3}, we conclude the theorem. 

\section{Proof of Theorem~\ref{thm:normal_form_lower}}
\label{app:proof_normal_form_lower}

We first suppose a mixed strategy $\pi:\mathsf{P}_{\Theta}\rightarrow\mathsf{P}_{a}$ for Player 1 misses at most $(1-\lambda)\left(\mu-\nu\right)$ opportunity, given a belief of types $\rho\in\mathsf{P}_{\Theta}$. 

Let $y'$ and $y''$ be two mixed strategies in $\Theta$ that achieve $\kappa(\Theta)\geq 0$ (select arbitrary strategies if there are multiple to break the tie), i.e.,
\begin{align*}
   \kappa(\Theta) = \sum_{i:y_i'\leq y_i''}y_i''-\sum_{i:y_i'>y_i''}y_i'' \text{ subject to } \sum_{i:y_i'\leq y_i''}y_i'< \sum_{i:y_i'>y_i''}y_i'.
\end{align*}
Let $\mathcal{I}_{+}\coloneq \{i\in [b] : y_i'>y_i'' \}$ and $\mathcal{I}_{-}\coloneq \{i\in [b] : y_i'\leq y_i'' \}$ be two indices of actions with non-negative and positive coordinates of $z$ where $[b]\coloneq \{1,\ldots,b\}$. Consider a payoff matrix $A_{\mu}$ with the following form. Suppose without loss of generality that $a$ is even; otherwise, we append an all-zero row to $A_{\mu}$. Let $\beta\coloneq \mu -\nu$. For the $i$-th column of $A_{\mu}$, we set it as $(\beta,-\beta,\ldots, -\beta)^{\top}$ if $i\in \mathcal{I}_{+}$ and $(-\beta,\beta,\ldots, \beta)^{\top}$ if $i\in \mathcal{I}_{-}$. 
\begin{align*}
  A=  A_{\mu}=\begin{bmatrix}
    \cdots  &  \beta & -\beta & \cdots  \\
    \cdots  & -\beta & \beta & \cdots  \\
     \vdots &  \vdots & \vdots & \vdots \\
    \cdots  &  \underbrace{-\beta}_{i\in \mathcal{I}_{+}} & \underbrace{\beta}_{i+1\in \mathcal{I}_{-}} & \cdots 
    \end{bmatrix} +  A_{\nu},
\end{align*}
where $(A_{\nu})_{ij}=\nu$ for all $i\in [a]$ and $j\in [b]$. Clearly, with this $A$, $\mu_{\Theta}(A)\leq \mu$ and $\nu_{\Theta}(A)\geq \nu$.

Then, by definition, setting $\varepsilon=0$ with $\mathbb{E}_{\rho}[y]=y^{\star}$ (the type belief contains no error), the payoff gap satisfies
\begin{align}
\label{eq:consistency_assumption}
\Delta_{\mathsf{NFG}}(0;\pi) =\max_{y^{\star}\in \Theta} \left(\max_{x\in \mathsf{P}_a}x^{\top} A y^{\star} - \pi(\rho) A y^{\star}\right) \leq (1-\lambda)(\mu-\nu).
\end{align}
% (Compared to:
% \begin{align*}
% \max_{(q,\widetilde{q}):d\left(q,\widetilde{q}\right)\leq 0} 2(2q-1)\left(p^*\left(q\right)-\pi\left(\widetilde{q}\right)\right) = \max_{q} 2(2q-1)\left(p^*\left(q\right)-\pi\left(q\right)\right) \leq 1-\lambda
% \end{align*})

Now, plugging in $y'$ into~\eqref{eq:consistency_assumption} above.
This implies 
\begin{align}
\label{eq:consistency_bound_1}
\max_{y^{\star}\in \Theta} \left(\max_{x\in \mathsf{P}_a}x^{\top} A y^{\star} - \pi(\rho)^{\top} A y^{\star}\right) \geq \max_{x\in \mathsf{P}_a}x^{\top} A y' - \pi(\rho)^{\top} A y' = \mu - \pi(\rho)^{\top} A y'.
\end{align}
For notational convenience, denote $\digamma(\Theta)\coloneq \sum_{i:y_i'>y_i''}y_i'-\sum_{i:y_i'\leq y_i''}y_i'$.
Combining~\eqref{eq:consistency_assumption} and~\eqref{eq:consistency_bound_1},
$
   \pi(\rho)^{\top} A y'=  \pi(\rho)^{\top} (A_{\mu}+A_{\nu}) y' \geq (1-\lambda) \nu +\lambda \mu.
$ Thus, $\pi(\rho)^{\top} A_{\mu} y'\geq \lambda (\mu-\nu)$. Equivalently,
$
    \digamma(\Theta)\left(\mu-\nu\right) \pi(\rho)^{\top} f  \geq \lambda \left(\mu-\nu\right),
$ where $f\coloneq (1,-1,1,\ldots,1,-1)^{\top}$.
By the construction of $A$, we get 
\begin{align*}
\pi(\rho)^{\top} A_{\mu} y''= -\left(\pi(\rho)^{\top} f \right)\left(\mu-\nu\right)\cdot\left(\sum_{i:y_i'\leq y_i''}y_i''-\sum_{i:y_i'>y_i''}y_i''\right) \leq - \lambda \left(\mu-\nu\right)\cdot\frac{\kappa(\Theta)}{\digamma(\Theta)}.
\end{align*}
Therefore, since $0<\digamma(\Theta)\leq 1$,
\begin{align*}
\pi(\rho)^{\top} (A_{\mu}+A_{\nu}) y'' = \nu + \pi(\rho)^{\top}A_{\mu} y'' \leq  \nu- \lambda \left(\mu-\nu\right)\kappa(\Theta) = -\lambda\mu\kappa(\Theta)+ (1+\lambda\kappa(\Theta))\nu.
\end{align*}
Hence, we get 
\begin{align*}
    \max_{y^{\star}\in \Theta} \left(\max_{x\in \mathsf{P}_a}x^{\top} A y^{\star} - \pi(\rho)^{\top} A y^{\star}\right) & \geq      \max_{x\in \mathsf{P}_a}x^{\top} A y'' - \pi(\rho)^{\top} A y'' \\
    & \geq \kappa(\Theta)\left(1+ \lambda\right)\mu- (1+\lambda\kappa(\Theta))\nu\\
    & \geq  \left(\kappa(\Theta)\mu-\nu\right)\left(1+\lambda\right).
    % & \geq (1+\lambda\kappa(\Theta))\left(\mu-\nu\right).
\end{align*}
% Then, plugging in this $y'$ and another $y''$ (to be constructed. Idea: consider a matrix $A$ with two columns $A_i$ and $A_j$ such that $A_i = -A_j$), we get
% (Compared to: \begin{align*}
% \max_{(q,\widetilde{q}):d\left(q,\widetilde{q}\right)\leq 1} 2(2q-1)\left(p^*\left(q\right)-\pi\left(\widetilde{q}\right)\right)\geq  2(1-\pi\left(0\right)) \geq 1+\lambda.
% \end{align*}
% \begin{align*}
% & \max_{(y,\widetilde{y}):\|y-\widetilde{y}\|\leq \varepsilon} \left(\max_{x\in\mathsf{P}_a}x^{\top}A y - \pi(\widetilde{y})^{\top}A y\right) \\
% \geq &\left(\max_{x\in\mathsf{P}_a}x^{\top}A y'' - \pi(y')^{\top}A y''\right)\geq (1+\lambda)\|A\|_{\max}-(1-\lambda)V
% \end{align*}

\section{Proof of Corollary~\ref{coro:nfg}}
\label{app:proof_normal_form_coro}

    If $\Theta=\mathsf{P}_b$, then $\kappa(\Theta)=1$. Furthermore, by definition, we get
    \begin{align}
\nonumber
    \mu_A\left(\mathsf{P}_b\right) \coloneq \max_{y\in\mathsf{P}_b}\max_{x\in\mathsf{P}_a}\left|x^{\top}A y\right|=\|A\|_{\max}, \quad 
    \nu_A\left(\mathsf{P}_b\right) \coloneq \min_{y\in\mathsf{P}_b}\max_{x\in\mathsf{P}_a}x^{\top}A y=0,
\end{align}
where the equality on the RHS holds since the game is fair. Applying Theorem~\ref{thm:normal_form_upper} and~\ref{thm:normal_form_lower}, we obtain the corollary.

\section{Proof of Theorem~\ref{thm:sbg_upper}}
\label{app:proof_sbg_upper}

Given type beliefs, we denote strategies $\widetilde{\sigma}\coloneq \sigma(\theta)$, $\sigma\coloneq \sigma\left(\theta^{\star}\right)$, and let $\overline{\sigma}=\sigma(\overline{\theta})$ be the safe strategies of the opponents, with $\overline{\theta}$ being an optimal solution of the minimax optimization in~\eqref{eq:sbg_minimiax}.
Denoting $\widetilde{V}\coloneq V^{\star,\widetilde{\sigma}}$, $\overline{V}\coloneq V^{\star,\overline{\sigma}}$, $V^{\star}\coloneq V^{\star,\sigma}$ the value functions with the opponents' strategies being $\widetilde{\sigma}$, $\overline{\sigma}$, and ${\sigma}$, respectively.
Given the strategy $\pi$ defined in~\eqref{eq:sbg_policy}, we first state two useful bounds.

% \begin{align}
% \label{eq:sbg_policy}    \pi\left(\widetilde{\sigma};s\right) \in \argmax_{a\in\Delta(\mathcal{A})} a^{\top} \left(\lambda R_{\widetilde{V}}(s)\widetilde{\sigma}(s)+(1-\lambda)R_{\overline{V}}(s)\overline{\sigma}(s)\right).
% \end{align}
% where for a given value function $V$, $R_{V}$ is a matrix defined as
% \begin{align*}
%     R_{V}((a,a');s) \coloneq r(s,(a,a')) + \gamma\sum_{s'\in\mathcal{S}}p(s'|s,(a,a'))V(s).
% \end{align*}

\begin{lemma}
\label{lemma:reward bound}
Given the strategy defined in~\eqref{eq:sbg_policy}, the following bounds on the reward function hold:
    \begin{align}
\nonumber
\textsc{Bound} \ \mathsf{1}: \quad \mathbb{E}_{\pi^{\star},\sigma}\left[r\right]-& \mathbb{E}_{\pi,\sigma}\left[r\right]\leq  \gamma\lambda\left(\mathbb{E}_{p,\pi,\widetilde{\sigma}}\big[V^{\star}(s)\big]-\mathbb{E}_{p,\pi^{\star},\widetilde{\sigma}}\big[V^{\star}(s)\big]\right) + 2\lambda\left(\varepsilon r_{\max}+\gamma\eta\right) \\
\nonumber
+ & \gamma(1-\lambda) \left(\mathbb{E}_{p,\pi, \overline{\sigma}}\left[\overline{V}(s)\right]-\mathbb{E}_{p,\pi^{\star},\overline{\sigma}}\left[\overline{V}(s)\right]\right)\\
     \label{eq:reward_bound_star}
+ & (1-\lambda) \left(\mathbb{E}_{\pi^{\star},\sigma}\left[r\right]-\mathbb{E}_{\pi^{\star},\overline{\sigma}}\left[r\right]\right) + (1-\lambda)\left(\mathbb{E}_{\pi,\overline{\sigma}}\left[r\right]-\mathbb{E}_{\pi,\sigma}\left[r\right]\right),
    \end{align}
where $\eta\coloneq \max_{s\in\mathcal{S}} \left|V^{\star}(s)-\widetilde{V}(s)\right|$,     and 
   \begin{align}
\nonumber \textsc{Bound} \ \mathsf{2}: \ 
\mathbb{E}_{\overline{\pi},\overline{\sigma}}\left[r\right]- \mathbb{E}_{\pi,\overline{\sigma}}&\left[r\right]   \leq \gamma  \left(\mathbb{E}_{p,\pi, \overline{\sigma}}\big[\overline{V}(s)\big]-\mathbb{E}_{p,\overline{\pi},\overline{\sigma}}\left[\overline{V}(s)\right]\right)\\
     \label{eq:reward_bound_overline}
+ \frac{\lambda}{1-\lambda}&\left(\left(\mathbb{E}_{\pi,\widetilde{\sigma}} \left[r\right]-\mathbb{E}_{\overline{\pi},\widetilde{\sigma}}\left[r\right]\right)+\gamma\left(\mathbb{E}_{p,\pi, \widetilde{\sigma}}\big[\widetilde{V}(s)\big]-\mathbb{E}_{p,\overline{\pi},\widetilde{\sigma}}\big[\widetilde{V}(s)\big]\right)\right).
    \end{align}
\end{lemma}

\begin{proof}[Proof of Lemma~\ref{lemma:reward bound}]

First, we prove $\textsc{Bound} \ \mathsf{1}$ in~\eqref{eq:reward_bound_star}.

Based on the definition of the policy $\pi$ in~\eqref{eq:sbg_policy}, it follows that for any fixed state $s\in\mathcal{S}$,
\begin{align}
\nonumber
& \lambda\mathbb{E}_{\pi^{\star},\widetilde{\sigma}}\left[r\right]+ (1-\lambda)\mathbb{E}_{\pi^{\star},\overline{\sigma}}\left[r\right]+ \gamma\lambda\mathbb{E}_{p,\pi^{\star},\widetilde{\sigma}}\big[\widetilde{V}(s)\big] + \gamma(1-\lambda)\mathbb{E}_{p,\pi^{\star},\overline{\sigma}}\big[\overline{V}(s)\big]\\
\label{eq:reward_bound_0}
    \leq & \lambda\mathbb{E}_{\pi,\widetilde{\sigma}}\left[r\right] +  (1-\lambda)\mathbb{E}_{\pi,\overline{\sigma}}\left[r\right] + \gamma\lambda\mathbb{E}_{p,\pi,\widetilde{\sigma}}\big[\widetilde{V}(s)\big] + \gamma(1-\lambda)\mathbb{E}_{p,\pi, \overline{\sigma}}\big[\overline{V}(s)\big].
\end{align}
Furthermore, considering the difference between the expected rewards corresponding to implementing $\pi^{\star}$ and the strategy $\pi$ defined in~\eqref{eq:sbg_policy}, since by Assumption~\ref{assumption:reward_bound}, $|r|\leq r_{\max}$ for all $t\in [\nt]$ and any strategy $\pi\in \Pi$,
\begin{align}
\label{eq:reward_bound_1}
\left|\mathbb{E}_{\pi,\sigma}\left[r\right] - \mathbb{E}_{\pi,\widetilde{\sigma}}\left[r\right]\right| \leq 
\max_{s\in\mathcal{S}}\left\|\sigma(s)-\widetilde{\sigma}(s)\right\|_1 r_{\max} \leq \varepsilon  r_{\max}. 
\end{align}
Therefore, using~\eqref{eq:reward_bound_1},
\begin{align}
\nonumber
\mathbb{E}_{\pi^{\star},\sigma}\left[r\right]-\mathbb{E}_{\pi,\sigma}\left[r\right]
= & \lambda\left(\mathbb{E}_{\pi^{\star},\sigma}\left[r\right]-\mathbb{E}_{\pi,\sigma}\left[r\right]\right) + (1-\lambda)\left(\mathbb{E}_{\pi^{\star},\sigma}\left[r\right]-\mathbb{E}_{\pi,\sigma}\left[r\right]\right) \\
\nonumber
\leq & \lambda \left(\mathbb{E}_{\pi^{\star},\widetilde{\sigma}}\left[r\right]-\mathbb{E}_{\pi,\widetilde{\sigma}}\left[r\right]\right) + (1-\lambda)\left(\mathbb{E}_{\pi^{\star},\overline{\sigma}}\left[r\right]-\mathbb{E}_{\pi,\overline{\sigma}}\left[r\right]\right) + 2\lambda \varepsilon r_{\max}\\
\nonumber
& \quad + (1-\lambda)\left(\mathbb{E}_{\pi^{\star},\sigma}\left[r\right]-\mathbb{E}_{\pi^{\star},\overline{\sigma}}\left[r\right]\right)\\
\nonumber
& \quad + (1-\lambda)\left(\mathbb{E}_{\pi,\overline{\sigma}}\left[r\right]-\mathbb{E}_{\pi,\sigma}\left[r\right]\right) \\
\nonumber
= &\left(\lambda\mathbb{E}_{\pi^{\star},\widetilde{\sigma}}\left[r\right]+ (1-\lambda)\mathbb{E}_{\pi^{\star},\overline{\sigma}}\left[r\right]\right) - \left(\lambda\mathbb{E}_{\pi,\widetilde{\sigma}}\left[r\right] +  (1-\lambda)\mathbb{E}_{\pi,\overline{\sigma}}\left[r\right] \right) \\
\nonumber
& \quad + (1-\lambda)\left(\mathbb{E}_{\pi^{\star},\sigma}\left[r\right]-\mathbb{E}_{\pi^{\star},\overline{\sigma}}\left[r\right]\right)\\
\label{eq:reward_bound_3}
& \quad + (1-\lambda)\left(\mathbb{E}_{\pi,\overline{\sigma}}\left[r\right]-\mathbb{E}_{\pi,\sigma}\left[r\right]\right) + 2\lambda\varepsilon r_{\max}.  
\end{align}

Combining~\eqref{eq:reward_bound_0} and~\eqref{eq:reward_bound_3}, we obtain
\begin{align}
\nonumber
& \mathbb{E}_{\pi^{\star},\sigma}\left[r\right]-\mathbb{E}_{\pi,\sigma}\left[r\right]-2\lambda \varepsilon r_{\max}\\
\nonumber
\leq & \gamma\left(\lambda\big(\mathbb{E}_{p,\pi,\widetilde{\sigma}}\big[\widetilde{V}(s)\big]-\mathbb{E}_{p,\pi^{\star},\widetilde{\sigma}}\big[\widetilde{V}(s)\big]\big)+(1-\lambda)\big(\mathbb{E}_{p,\pi, \overline{\sigma}}\big[\overline{V}(s)\big]-\mathbb{E}_{p,\pi^{\star},\overline{\sigma}}\big[\overline{V}(s)\big]\big)\right)\\
\label{eq:value_difference_1}
& \quad + (1-\lambda)\left(\mathbb{E}_{\pi^{\star},\sigma}\left[r\right]-\mathbb{E}_{\pi^{\star},\overline{\sigma}}\left[r\right]\right) + (1-\lambda)\left(\mathbb{E}_{\pi,\overline{\sigma}}\left[r\right]-\mathbb{E}_{\pi,\sigma}\left[r\right]\right).
 \end{align}
By definition,
\begin{align*}
\left|\mathbb{E}_{p,\pi,\widetilde{\sigma}}\big[V^{\star}(s)\big]-\mathbb{E}_{p,\pi,\widetilde{\sigma}}\big[\widetilde{V}(s)\big]\right|\leq \sum_{s'\in\mathcal{S}}\sum_{a,a'\in\mathcal{A}}   p(s'|s,(a,a'))\pi_t(s,a)\widetilde{\sigma}_t(s,a')\eta\leq \eta,
\end{align*}
and similarly, $\left|\mathbb{E}_{p,\pi^{\star},\widetilde{\sigma}}\big[V^{\star}(s)\big]-\mathbb{E}_{p,\pi^{\star},\widetilde{\sigma}}\big[\widetilde{V}(s)\big]\right|\leq \eta$,
thus,~\eqref{eq:value_difference_1} yields 
\begin{align*}
\mathbb{E}_{\pi^{\star},\sigma}\left[r\right]-\mathbb{E}_{\pi,\sigma}\left[r\right]\leq  & \gamma\lambda\left(\mathbb{E}_{p,\pi,\widetilde{\sigma}}\big[V^{\star}(s)\big]-\mathbb{E}_{p,\pi^{\star},\widetilde{\sigma}}\big[V^{\star}(s)\big]\right) + 2\lambda\left(\varepsilon r_{\max}+\gamma\eta\right) \\
 & + \gamma(1-\lambda)\left(\mathbb{E}_{p,\pi, \overline{\sigma}}\left[\overline{V}(s)\right]-\mathbb{E}_{p,\pi^{\star},\overline{\sigma}}\left[\overline{V}(s)\right]\right)\\
& + (1-\lambda)\left(\mathbb{E}_{\pi^{\star},\sigma}\left[r\right]-\mathbb{E}_{\pi^{\star},\overline{\sigma}}\left[r\right]\right) + (1-\lambda)\left(\mathbb{E}_{\pi,\overline{\sigma}}\left[r\right]-\mathbb{E}_{\pi,\sigma}\left[r\right]\right).
 \end{align*}

Next, we show $\textsc{Bound} \ \mathsf{2}$ in~\eqref{eq:reward_bound_overline}.

By the definition of the policy $\pi$ in~\eqref{eq:sbg_policy}, it follows that for any fixed state $s\in\mathcal{S}$, similarly we obtain
\begin{align}
\nonumber
& \lambda\mathbb{E}_{\overline{\pi},\widetilde{\sigma}}\left[r\right]+ (1-\lambda)\mathbb{E}_{\overline{\pi},\overline{\sigma}}\left[r\right]+ \gamma\lambda\mathbb{E}_{p,\overline{\pi},\widetilde{\sigma}}\big[\widetilde{V}(s)\big] + \gamma(1-\lambda)\mathbb{E}_{p,\overline{\pi},\overline{\sigma}}\big[\overline{V}(s)\big]\\
\label{eq:reward_bound_2}
    \leq & \lambda\mathbb{E}_{\pi,\widetilde{\sigma}}\left[r\right] +  (1-\lambda)\mathbb{E}_{\pi,\overline{\sigma}}\left[r\right] + \gamma\lambda\mathbb{E}_{p,\pi,\widetilde{\sigma}}\big[\widetilde{V}(s)\big] + \gamma(1-\lambda)\mathbb{E}_{p,\pi, \overline{\sigma}}\big[\overline{V}(s)\big].
\end{align}

Rearranging the terms in~\eqref{eq:reward_bound_2},
\begin{align}
    \nonumber
\mathbb{E}_{\overline{\pi},\overline{\sigma}}\left[r\right]- \mathbb{E}_{\pi,\overline{\sigma}}&\left[r\right]   \leq \gamma  \left(\mathbb{E}_{p,\pi, \overline{\sigma}}\left[\overline{V}(s)\right]-\mathbb{E}_{p,\overline{\pi},\overline{\sigma}}\left[\overline{V}(s)\right]\right)\\
\nonumber
+ \frac{\lambda}{1-\lambda}&\left(\left(\mathbb{E}_{\pi,\widetilde{\sigma}} \left[r\right]-\mathbb{E}_{\overline{\pi},\widetilde{\sigma}}\left[r\right]\right)+\gamma\left(\mathbb{E}_{p,\pi, \widetilde{\sigma}}\big[\widetilde{V}(s)\big]-\mathbb{E}_{p,\overline{\pi},\widetilde{\sigma}}\big[\widetilde{V}(s)\big]\right)\right) .
\end{align}
 \end{proof}

\begin{lemma}
\label{lemma:strategy_bound}
The payoff gap for the policy $\pi$ defined in~\eqref{eq:sbg_policy} satisfies
\begin{align*}
\Delta_{\mathsf{SBG}}(\varepsilon;\pi) \leq \frac{r_{\max}}{1-\lambda\gamma}\Bigg(\frac{2\lambda\varepsilon}{1-\gamma} & +\min\left\{\frac{2(1-\lambda)}{1-\gamma},\frac{2\lambda}{(1-\gamma)^2}\right\}\\
& +\frac{(1-\lambda)(2-\gamma)}{1-\gamma}\Bigg) +\frac{\gamma\left(2\lambda\eta-\nu\right)}{1-\lambda\gamma}.
\end{align*}
\end{lemma}

\begin{proof}
By definition of the payoff gap (see Definition~\ref{def:payoff_sbg}) and the value function in~\eqref{eq:value_function_pi}, we get
\begin{align*}
    \Delta_{\mathsf{SBG}}(\varepsilon;\pi)=\sup_{d\left(\theta,\theta^{\star}\right)\leq \varepsilon}\max_{s\in\mathcal{S}} \left(V^{\star,\sigma}(s)-V^{\pi(\theta),\sigma}(s)\right).
\end{align*}

% We extend the arguments in~\cite{singh1994upper}.
Our first goal is to derive an upper bound on $\Delta_{\mathsf{SBG}}(\varepsilon;\pi)$, aligning with the classic loss bound of approximate value functions~\cite{singh1994upper}. However, the policy considered in our context is not directly maximizing the approximate value function, but a mixed strategy defined by an optimization as in~\eqref{eq:sbg_policy}.
For notational convenience, we write for $s\in \mathcal{S}$,
\begin{align*}
\mathbb{E}_{p,\pi,\sigma}\left[V(s)\right]\coloneq
\sum_{a,a'\in \mathcal{A}}\sum_{s'\in\mathcal{S}} p(s'|s,(a,a'))\pi(s,a)\sigma(s,a')V(s').
\end{align*}

For notational simplicity, we denote $V^{\pi}(s)\coloneq V^{\pi,\sigma}(s)$. For any state $s\in\mathcal{S}$ and strategies $\left(\sigma,\widetilde{\sigma}\right)$ satisfying $d\left(\theta,\theta^{\star}\right)=\max_{s\in\mathcal{S}}\left\|\sigma(s;\theta)-\sigma(s;\theta^{\star})\right\|_1\leq \varepsilon$, applying $\textsc{Bound} \ \mathsf{1}$ in Lemma~\ref{lemma:reward bound}, the Bellman equations corresponding to $\pi^{\star}$ and $\pi$ imply
\begin{align}
\nonumber
V^{\star}(s)-V^{\pi}(s)
= & \mathbb{E}_{\pi^{\star},\sigma}[r]-\mathbb{E}_{\pi,\sigma}[r] +\gamma \left(\mathbb{E}_{p,\pi^{\star},\sigma}\big[V^{\star}(s)\big]-\mathbb{E}_{p,\pi,\sigma}\big[V^{\pi}(s)\big]\right) \\
\nonumber
\leq \lambda  \gamma 
\Big(\mathbb{E}_{p,\pi,\widetilde{\sigma}} & \big[V^{\star}(s)\big] -\mathbb{E}_{p,\pi^{\star},\widetilde{\sigma}}\big[V^{\star}(s)\big]+ \mathbb{E}_{p,\pi^{\star},\sigma}\big[V^{\star}(s)\big]-\mathbb{E}_{p,\pi,\sigma}\big[V^{\pi}(s)\big]\Big)
\\ 
\nonumber
+
\gamma(1-\lambda)  &\left(\mathbb{E}_{p,\pi, \overline{\sigma}}\left[\overline{V}(s)\right]-\mathbb{E}_{p,\pi,\sigma}\big[V^{\pi}(s)\big] + \mathbb{E}_{p,\pi^{\star},\sigma}\big[V^{\star}(s)\big]-\mathbb{E}_{p,\pi^{\star},\overline{\sigma}}\big[\overline{V}(s)\big]\right)\\
\nonumber
  + (1-\lambda)&\left(\mathbb{E}_{\pi,\overline{\sigma}}\left[r\right]-\mathbb{E}_{\pi,\sigma}\left[r\right]\right) + (1-\lambda)\left(\mathbb{E}_{\pi^{\star},\sigma}\left[r\right]-\mathbb{E}_{\pi^{\star},\overline{\sigma}}\left[r\right]\right)\\
\label{eq:value_bound_0}
 + 2\lambda & \left(\varepsilon r_{\max}+\gamma\eta\right).
\end{align}

Applying the Bellman equations, for any $s\in\mathcal{S}$,
\begin{align}
\nonumber
V^{\pi}(s) = & \mathbb{E}_{a\sim\pi(s),a'\sim\sigma}(s) \left[\left(r + \gamma\mathbb{P} V^{\pi}\right)(s,(a,a'))\right]\\
\label{eq:value_function_pi_1}
=&\mathbb{E}_{\pi,\sigma}\left[r\right]+ \gamma\mathbb{E}_{p,\pi,\sigma}\left[V^{\pi}(s)\right].
\end{align}
Similarly, it also holds that
\begin{align}
\label{eq:value_function_star}
V^{\star}(s) &= \mathbb{E}_{\pi^{\star},{\sigma}}\left[r\right]+ \gamma\mathbb{E}_{p,\pi^{\star},\sigma}\left[V^{\star}(s)\right].
\end{align}

Plugging~\eqref{eq:value_function_pi_1}, and~\eqref{eq:value_function_star} into~\eqref{eq:value_bound_0},
\begin{align}
\nonumber
V^{\star}(s)-V^{\pi}(s)
\leq & \lambda  \gamma 
\Big(\mathbb{E}_{p,\pi,\widetilde{\sigma}}  \big[V^{\star}(s)\big] -\mathbb{E}_{p,\pi^{\star},\widetilde{\sigma}}\big[V^{\star}(s)\big]+ \mathbb{E}_{p,\pi^{\star},\sigma}\big[V^{\star}(s)\big]-\mathbb{E}_{p,\pi,\sigma}\big[V^{\pi}(s)\big]\Big)\\
\label{eq:value_bound_2}
& +  (1-\lambda)\Big(\left(\left(\mathbb{E}_{\pi,\overline{\sigma}}\left[r\right]+\gamma\mathbb{E}_{p,\pi,\overline{\sigma}}\left[\overline{V}(s)\right]\right) -V^{\pi}(s)\right)\Big)\\
\label{eq:value_bound_3}
& +  (1-\lambda)\Big(V^{\star}(s)-\left(\mathbb{E}_{\pi^{\star},\overline{\sigma}}\left[r\right]+\gamma\mathbb{E}_{p,\pi^{\star},\overline{\sigma}}\left[\overline{V}(s)\right]\right)\Big)\\
\label{eq:value_bound_1}
& + 2\lambda\left(\varepsilon r_{\max}+\gamma\eta\right).
\end{align}

\textbf{Part I: Proof of Risk Bound}

Let $\overline{V}^{\pi}(s)\coloneq V^{\pi,\overline{\sigma}}(s)$ for any $s\in\mathcal{S}$.
For the term in~\eqref{eq:value_bound_2}, we always have $V^{\pi}(s)\geq \overline{V}^{\pi}(s)$~\cite{shapley1953stochastic},
\begin{align*}
\left(\mathbb{E}_{\pi,\overline{\sigma}}\left[r\right]+\gamma\mathbb{E}_{p,\pi,\overline{\sigma}}\left[\overline{V}(s)\right]\right) -V^{\pi}(s) 
\leq & \gamma\left(\mathbb{E}_{p,\pi,\overline{\sigma}}\left[\overline{V}(s)\right]-\mathbb{E}_{p,\pi,\overline{\sigma}}\big[\overline{V}^{\pi}(s)\big]\right)\\
= &\gamma\mathbb{E}_{p,\pi,\overline{\sigma}}\left[\overline{V}(s)-\overline{V}^{\pi}(s)\right].
\end{align*}
Furthermore, the Bellman equation implies
\begin{align}
\label{eq:bellman_bound_1}
\overline{V}(s)-\overline{V}^{\pi}(s) =\left(\mathbb{E}_{\overline{\pi},\overline{\sigma}}\left[r\right]- \mathbb{E}_{\pi,\overline{\sigma}}\left[r\right] \right)+\gamma\left(\mathbb{E}_{p,\overline{\pi},\overline{\sigma}}\left[\overline{V}(s)\right]-\mathbb{E}_{p,\pi,\overline{\sigma}}\big[\overline{V}^{\pi}(s)\big]\right).
\end{align}
Applying $\textsc{Bound} \ \mathsf{2}$ in Lemma~\ref{lemma:reward bound},
\begin{align}
\nonumber
\overline{V}(s)-\overline{V}^{\pi}(s) \leq & \gamma\left(\mathbb{E}_{p,\overline{\pi},\overline{\sigma}}\left[\overline{V}(s)\right]-\mathbb{E}_{p,\pi,\overline{\sigma}}\big[\overline{V}^{\pi}(s)\big]+ \mathbb{E}_{p,\pi, \overline{\sigma}}\left[\overline{V}(s)\right]-\mathbb{E}_{p,\overline{\pi},\overline{\sigma}}\left[\overline{V}(s)\right]\right)\\
\nonumber
& +  \frac{\lambda}{1-\lambda}\left(\left(\mathbb{E}_{\pi,\widetilde{\sigma}} \left[r\right]-\mathbb{E}_{\overline{\pi},\widetilde{\sigma}}\left[r\right]\right)+\gamma\left(\mathbb{E}_{p,\pi, \widetilde{\sigma}}\big[\widetilde{V}(s)\big]-\mathbb{E}_{p,\overline{\pi},\widetilde{\sigma}}\big[\widetilde{V}(s)\big]\right)\right)\\
\nonumber
\leq & \gamma\left(\mathbb{E}_{p,\pi, \overline{\sigma}}\left[\overline{V}(s)\right]-\mathbb{E}_{p,\pi,\overline{\sigma}}\big[\overline{V}^{\pi}(s)\big]\right)\\
\label{eq:bellman_bound_2}
& + \frac{2\lambda}{1-\lambda}\left(\left(1+\frac{\gamma}{1-\gamma}\right)\right)r_{\max}.
\end{align}
Since above holds for any state $s\in\mathcal{S}$, let $s'$ be a state such that the value gap $V^{\star}(s)-V^{\pi}(s)$ is maximized. Therefore,~\eqref{eq:bellman_bound_2} implies
\begin{align}
\label{eq:bellman_bound_3}
\overline{V}(s)-\overline{V}^{\pi}(s) \leq \frac{2}{(1-\gamma)^2}\frac{\lambda}{1-\lambda}r_{\max}.
\end{align}

Then, for the term in~\eqref{eq:value_bound_3}, by definition $V^{\star}(s)\leq r_{\max}/(1-\gamma)$, and 
\begin{align}
\label{eq:value_bound_5}
\mathbb{E}_{\pi^{\star},\overline{\sigma}}\left[r\right]+\gamma\mathbb{E}_{p,\pi^{\star},\overline{\sigma}}\left[\overline{V}(s)\right]\geq \gamma\nu - r_{\max}.
\end{align}

Combing the bounds in~\eqref{eq:bellman_bound_3} and~\eqref{eq:value_bound_5} with ~\eqref{eq:value_bound_1}, we conclude that
\begin{align}
\nonumber
V^{\star}(s)-V^{\pi}(s)
\leq & \lambda  \gamma 
\Big(\mathbb{E}_{p,\pi,\widetilde{\sigma}}  \big[V^{\star}(s)\big] -\mathbb{E}_{p,\pi^{\star},\widetilde{\sigma}}\big[V^{\star}(s)\big]+ \mathbb{E}_{p,\pi^{\star},\sigma}\big[V^{\star}(s)\big]-\mathbb{E}_{p,\pi,\sigma}\big[V^{\pi}(s)\big]\Big)\\
\label{eq:value_bound_4}
& + \frac{2}{(1-\gamma)^2}\lambda r_{\max} + (1-\lambda)\left(\frac{2-\gamma}{1-\gamma}r_{\max} - \gamma\nu\right) + 2\lambda  \left(\varepsilon r_{\max}+\gamma\eta\right).
\end{align}

Now, noting that by definition we have
\begin{align*}
&\left|\mathbb{E}_{p,\pi,\sigma}\big[V^{\star}\big]-\mathbb{E}_{p,\pi,\widetilde{\sigma}}\big[V^{\star}\big]\right|\\
\leq & \sum_{a,a'\in \mathcal{A},s'\in\mathcal{S}} p(s'|s,(a,a'))\pi(s,a)\left|\sigma(s,a')-\widetilde{\sigma}(s,a')\right|V^{\star}(s')\leq \varepsilon \frac{r_{\max}}{1-\gamma},
\end{align*}
and similarly,
\begin{align*}
\left|\mathbb{E}_{p,\pi^{\star},\sigma}\big[V^{\star}\big]-\mathbb{E}_{p,\pi^{\star},\widetilde{\sigma}}\big[V^{\star}\big]\right|
\leq & \varepsilon \frac{r_{\max}}{1-\gamma}.
\end{align*}

Therefore, rearranging the terms and simplifying above,~\eqref{eq:value_bound_4} becomes
\begin{align}
\nonumber
V^{\star}(s)-V^{\pi}(s)
\leq & \frac{2\lambda}{(1-\gamma)^2} r_{\max} +(1-\lambda)\left(\frac{2-\gamma}{1-\gamma}r_{\max} - \gamma\nu\right)+ 2\lambda  \left(\varepsilon r_{\max}+\gamma\eta\right) \\
\label{eq:value_gap_1}
& + \frac{2\gamma\lambda r_{\max}}{1-\gamma}\varepsilon + \gamma\lambda\left( \mathbb{E}_{p,\pi,\sigma}\big[V^{\star}\big] - \mathbb{E}_{p,\pi,\sigma}[V^{\pi}]\right).
\end{align}
Since above holds for any state $s\in\mathcal{S}$, let $s^*$ be a state such that the value gap $V^{\star}(s)-V^{\pi}(s)$ is maximized. Thus,
\begin{align*}
&\mathbb{E}_{p,\pi,\sigma}\big[V^{\star}(s^*)\big] - \mathbb{E}_{p,\pi,\sigma}[V^{\pi}(s^*)] \\
=& \sum_{a,a'\in \mathcal{A}}\sum_{s'\in\mathcal{S}} p\left(s'|s^*,(a,a')\right)\pi\left(s^*,a\right)\sigma\left(s^*,a'\right)\left(V^{\star}(s')-V^{\pi}(s')\right)
\leq  V^{\star}(s^*)-V^{\pi}(s^*).
\end{align*}

Continuing from~\eqref{eq:value_gap_1},
\begin{align*}
& \left(V^{\star}(s^*)-V^{\pi}(s^*)\right) - \lambda\gamma \left(V^{\star}(s^*)-V^{\pi}(s^*)\right)\\
\leq & 2r_{\max}\left(\frac{\lambda}{(1-\gamma)^2}+\frac{\gamma\lambda}{1-\gamma}\varepsilon+\lambda\varepsilon\right)+ 2\lambda\gamma\eta
 + (1-\lambda)\left(\frac{2-\gamma}{1-\gamma}r_{\max} - \gamma\nu\right)\\
= & 2r_{\max}\left(\frac{\lambda}{(1-\gamma)^2}+\frac{\lambda\varepsilon}{1-\gamma}\right) + 2\lambda\gamma\eta + (1-\lambda) \left(\frac{2-\gamma}{1-\gamma}r_{\max} - \gamma\nu\right).
\end{align*}
Therefore, rearranging the terms we obtain
\begin{align*}
\Delta_{\mathsf{SBG}}(\varepsilon;\pi) \leq \frac{r_{\max}}{1-\lambda\gamma}\left(\frac{2\lambda\varepsilon}{1-\gamma}+\frac{2\lambda}{(1-\gamma)^2}+\frac{(1-\lambda)(2-\gamma)}{1-\gamma}\right)+\frac{\gamma\left(2\lambda\eta-(1-\lambda)\nu\right)}{1-\lambda\gamma}.
\end{align*}

\textbf{Part II: Proof of Opportunity Bound}

Now, the terms in~\eqref{eq:value_bound_2} can be bounded alternatively as
\begin{align*}
(1-\lambda)\Big(\left(\left(\mathbb{E}_{\pi,\overline{\sigma}}\left[r\right]+\gamma\mathbb{E}_{p,\pi,\overline{\sigma}}\left[\overline{V}(s)\right]\right) -V^{\pi}(s)\right)\Big)\leq & 2\frac{1-\lambda}{1-\gamma}r_{\max}.
\end{align*}
Using this and following the same steps, we obtain
\begin{align*}
\Delta_{\mathsf{SBG}}(\varepsilon;\pi) \leq \frac{r_{\max}}{1-\lambda\gamma}\left(\frac{2\lambda\varepsilon}{1-\gamma}+\frac{2(1-\lambda)}{(1-\gamma)^2}+\frac{(1-\lambda)(2-\gamma)}{1-\gamma}\right)+\frac{\gamma\left(2\lambda\eta-(1-\lambda)\nu\right)}{1-\lambda\gamma}.
\end{align*}

\end{proof}

Finally, setting $\varepsilon=\eta=0$, we obtain 
\begin{align*}
\Delta_{\mathsf{SBG}}(0;\pi) 
% \leq & \frac{r_{\max}}{1-\lambda\gamma}\left(\frac{2\lambda\varepsilon}{1-\gamma}+\frac{2(1-\lambda)}{(1-\gamma)^2}+\frac{(1-\lambda)(2-\gamma)}{1-\gamma}\right)+\frac{\gamma\left(2\lambda\eta-(1-\lambda)\nu\right)}{1-\lambda\gamma},\\
\leq &\frac{1}{1-\lambda\gamma} \left(r_{\max}\left(\frac{\gamma^2-3\gamma+6}{(1-\gamma)^2}\right)-\gamma\nu\right)(1-\lambda).
\end{align*}

Since $d\left(\theta,\theta^{\star}\right)\coloneq\max_{s\in\mathcal{S}}\left\|\sigma(s;\theta)-\sigma(s;\theta^{\star})\right\|_1\leq 2$, setting the worst $\varepsilon=2$, and noticing that $\eta\leq \frac{2 r_{\max}}{1-\gamma}$,
\begin{align*}
   \Delta_{\mathsf{SBG}}(\varepsilon;\pi)
   % \leq &\frac{r_{\max}}{1-\lambda\gamma}\left(\frac{2\lambda\varepsilon}{1-\gamma}+\frac{2\lambda}{(1-\gamma)^2}+\frac{(1-\lambda)(2-\gamma)}{1-\gamma}\right)+\frac{\gamma\left(2\lambda\eta-(1-\lambda)\nu\right)}{1-\lambda\gamma}\\
   \leq & \frac{r_{\max}}{1-\lambda\gamma}\left(\frac{\left(4(1-\gamma^2)+2\right)\lambda}{(1-\gamma)^2}+\frac{(1-\lambda)(2-\gamma)}{1-\gamma}\right) - \frac{\gamma(1-\lambda)\nu}{1-\lambda\gamma}.
\end{align*}

\section{Proof of Theorem~\ref{thm:sbg_lower}}
\label{app:proof_sbg_lower}

We prove the theorem by considering a stateless MDP, which reduces a stochastic Bayesian game $\mathcal{M}= \langle \mathcal{S}, \mathcal{A}, \Theta, \sigma, p, r, \gamma\rangle$ to a normal-form game in Theorem~\ref{thm:normal_form_lower}. Furthermore, let $\mathcal{L}\coloneq \left\{1,\ldots,\prod_{j\neq i}|\mathcal{A}(j)\right\}$. We fix a strategy kernel $\sigma$ such that
$$\kappa(\Theta)\coloneq \max_{s\in\mathcal{S}, \theta,\theta'\in\Theta}\left(\sum_{l\in\mathcal{L}:\sigma_l(s;\theta)=0}\sigma_l(s;\theta')-\sum_{l\in\mathcal{L}:\sigma_l(s;\theta)>0}\sigma_l(s;\theta')\right)=1.$$
Therefore, the construction of the payoff matrix $A\in\mathbb{R}^{|\mathcal{A}(i)|\times \prod_{j\neq i}|\mathcal{A}(j)|}$ in the proof of Theorem~\ref{thm:normal_form_lower} with $\mu_{\Theta}(A)\leq \frac{r_{\max}}{1-\gamma}$ and $\nu_{\Theta}(A)\geq \nu$ 
 yields that for any $0\leq \lambda\leq 1$, if any mixed strategy $\pi$ misses $\smash{\frac{1}{1-\lambda\gamma}\left(r_{\max}-\nu\right)(1-\lambda)}$ opportunity, then since
 \begin{align*}
\frac{1}{1-\lambda\gamma}\left(r_{\max}-\nu\right)(1-\lambda) \leq \left(\frac{r_{\max}}{1-\gamma}-\nu\right)(1-\lambda),
 \end{align*} then $\pi$ has at least $\left(\frac{r_{\max}}{1-\gamma}-\nu\right)\left(1+\lambda\right)$ risk. Furthermore,
 since
 \begin{align*}
\left(\frac{r_{\max}}{1-\gamma}-\nu\right)\left(1+\lambda\right) \geq \left(\frac{r_{\max}}{1-\lambda\gamma}-\frac{\nu}{1-\lambda\gamma}\right)\left(1+\lambda\right),
 \end{align*}
 $\pi$ has at least $\frac{1}{1-\lambda\gamma}\left(r_{\max}-\nu\right)(1+\lambda)$ risk.
\section{Details on Experiments}
\label{app:experiments}
\subsection{Details on $2 \times 2$ Game simulations}

\subsubsection{Game definition}
\begin{itemize}
    \item We define a Stochastic Bayesian game with one state: a particular $2 \times 2$ game sampled from the topology of $2 \times 2$ games. 
    \item The payoffs for each player are stipulated by the $2 \times 2$ game. We set the time horizon for the game be 1,000 giving us an empirical evaluation of the expected payoff for each strategy
\end{itemize}
\subsubsection{Type definitions}
As mentioned we have 3 classes of types and below we provide a description of the general classes as well as the specific types we used in our evaluation. 
\begin{enumerate}
    \item \textbf{Markovian Types:} These are types whose strategy only depends on the current state. We made use of 4 types of Markovian strategies
    \begin{itemize}
        \item \textbf{Type 1:} Always play action 0
        \item \textbf{Type 2:} Always play action 1
        \item \textbf{Type 3:} The minimax strategy for the player
        \item \textbf{Type 4:} Returns a random strategy
    \end{itemize}
    \item \textbf{Leader- Follower-Trigger Agents:} Inspired by \cite{crandall2014towards}, these agents have a preferred sequence of play they seek to enforce. Importantly they have access to history of play and when the other player does not play according to their preferred sequence the alter their strategy by engaging in ``punishing'' behavior such as playing the minimax strategy or by resetting to a previous action. In our case, we evaluated against simple versions of such agents, wherein our agent had a preferred mixed strategy and when the empirical observed strategy over a preset number of previous plays from the opponent did not match their preferred strategy, they ``punished'' the opponent by playing a minimax strategy. In particular, we had an agent look at the previous 4 actions of the opponent and if they selected action 1 more than twice, they chose to play a minimax strategy for the next round. 
    \item \textbf{Co-evolved Neural Networks:} Inspired by work in \cite{albrecht2015belief} we use ideas of genetic programming \cite{koza1992genetic}, to generate agents from neural networks.  We randomly initialized 10 neural networks with a single hidden layer for both the row and column player. All networks would take as input the previous 4 actions of both players and the corresponding state information. We sample randomly from populations of the row and column players and simulate the game. After simulation we calculate a fitness score based on average payoffs for each agent and a similarity score so as to ensure diversity in the models. We ``evolve'' the populations by selecting random portions of both populations to mutate whilst also having cross-over between members of the populations selected by fitness (this is done for both populations hence ``co-evolve''). We then proceed to create new populations using the most fit agents from a previous generation (i.e., we take the top 50\% of a population pre-evolution and 50\% post evolution and then create a new population if the average fitness of this new constructed population is greater than the average fitness of the previous population.)
\end{enumerate}
\subsubsection{Experimental illustration of tightness of bounds}
We have included additional adversarial examples to demonstrate the tightness of the bounds presented in Theorems ~\ref{thm:normal_form_upper} and \ref{thm:normal_form_lower}. In particular, we consider the zero-sum Matching Pennies (MP) and an Adjusted Matching Pennies (AMP) as our two examples. We assume the hypothesis set $\Theta$ for Player 2 contains 6 behavioral types:
\begin{itemize}
    \item \textbf{Markovian Types}:
    \begin{itemize}
        \item \textbf{Type 1}: Always play action 0
        \item \textbf{Type 2}: Always play action 1
        \item \textbf{Type 3}: Minimax strategy
        \item \textbf{Type 4}: Random strategy
    \end{itemize}
    \item \textbf{Type 5}: Leader-Follower-Trigger Agents
    \item \textbf{Type 6}: Co-evolved Neural Networks
\end{itemize}

The payoff matrices for the MP and AMP, respectively are
\begin{align*}
    \begin{bmatrix} 1 & -1 \\ -1 & 1 \end{bmatrix}, \text{and }\begin{bmatrix} 1.2 & -0.8 \\ -0.8 & 1.2 \end{bmatrix}=\begin{bmatrix} 1 & -1 \\ -1 & 1 \end{bmatrix} + \begin{bmatrix} 0.2 & 0.2 \\ 0.2 & 0.2 \end{bmatrix}. 
\end{align*}

It's worth mentioning that the AMP payoff matrix constructed above is an adversarial example that follows the same construction of the adversarial payoff matrix $A$ used in the proof of Theorem \ref{thm:normal_form_lower}. Given the considered $\Theta$, we know that $\kappa(\Theta)=1$ and $\eta(\Theta)=2$. The upper and lower bounds in Theorem \ref{thm:normal_form_upper} and \ref{thm:normal_form_lower} therefore read:
\begin{align*}
    &\textit{Upper Bound: } x=(1-\lambda)(\mu-\nu), \ y=(1-\lambda)(\mu-\nu)+2\lambda\mu, \lambda\in [0,1],\\
    &\textit{Lower Bound: } x=(1-\lambda)(\mu-\nu), \ y=(1+\lambda)(\mu-\nu), \lambda\in [0,1],
\end{align*}
where for MP, $\mu=1$, $\nu=0$; for AMP, $\mu=1.2$, $\nu=0.2$.

Besides the bounds above, we also plot simulated risk and opportunity values in the figures using an algorithm that has varying trust of type beliefs in 100 and 1,000 runs. The algorithm is presented with a type prediction and takes a convex combination of the best response to the type prediction and the minimax strategy with the trust parameter $\lambda$ determining how much to weigh the best response. It is the same as $\pi$, the mixed strategy used to prove Theorem \ref{thm:normal_form_upper}. Fully trusting ($\lambda=1$) and distrusting ($\lambda=0$) type beliefs yield a best response strategy and a minimax strategy correspondingly. The agent then samples from their resulting mixed strategy an action to play while the opponent also samples from whatever mixed strategy they are using an action to play.

We sample this interaction for the number of runs and gather empirical payoff information, which we use to plot the opportunity risk tradeoff. The variance of the risk and opportunity values decreases as the number of runs increases, as illustrated by the error bars in the attached figures. Moreover, we include as plots the bounds on the opportunity risk tradeoff and show the tightness or looseness of these bounds in two games we pick. Please note that some of the simulated risk and opportunity values fall outside the bounds. This is because the bounds are applicable only to expected payoff gaps, and individual simulations may deviate from these expected values.

In the Fig \ref{fig:bounds}, the top two figures display the results from 100 runs, while the bottom two figures present the results from 1,000 runs. The left two figures correspond to the MP example, while the right two figures are for the AMP example, demonstrating the gap between the lower and upper bounds. In particular, in an AMP game, we see that the upper bound is loose, and the empirical opportunity-risk tradeoff matches the lower bound we derive. In future work, we will investigate if this gap can be closed. This ‘adversarial’ example illustrates the looseness in the upper bound we derive. In the canonical Matching Pennies game, we find that the lower and upper bounds are tight with the empirical tradeoff matching these theoretical bounds, coinciding with Corollary \ref{coro:nfg}.

\begin{figure}[h]
    \centering
    \includegraphics[width=0.8\textwidth]{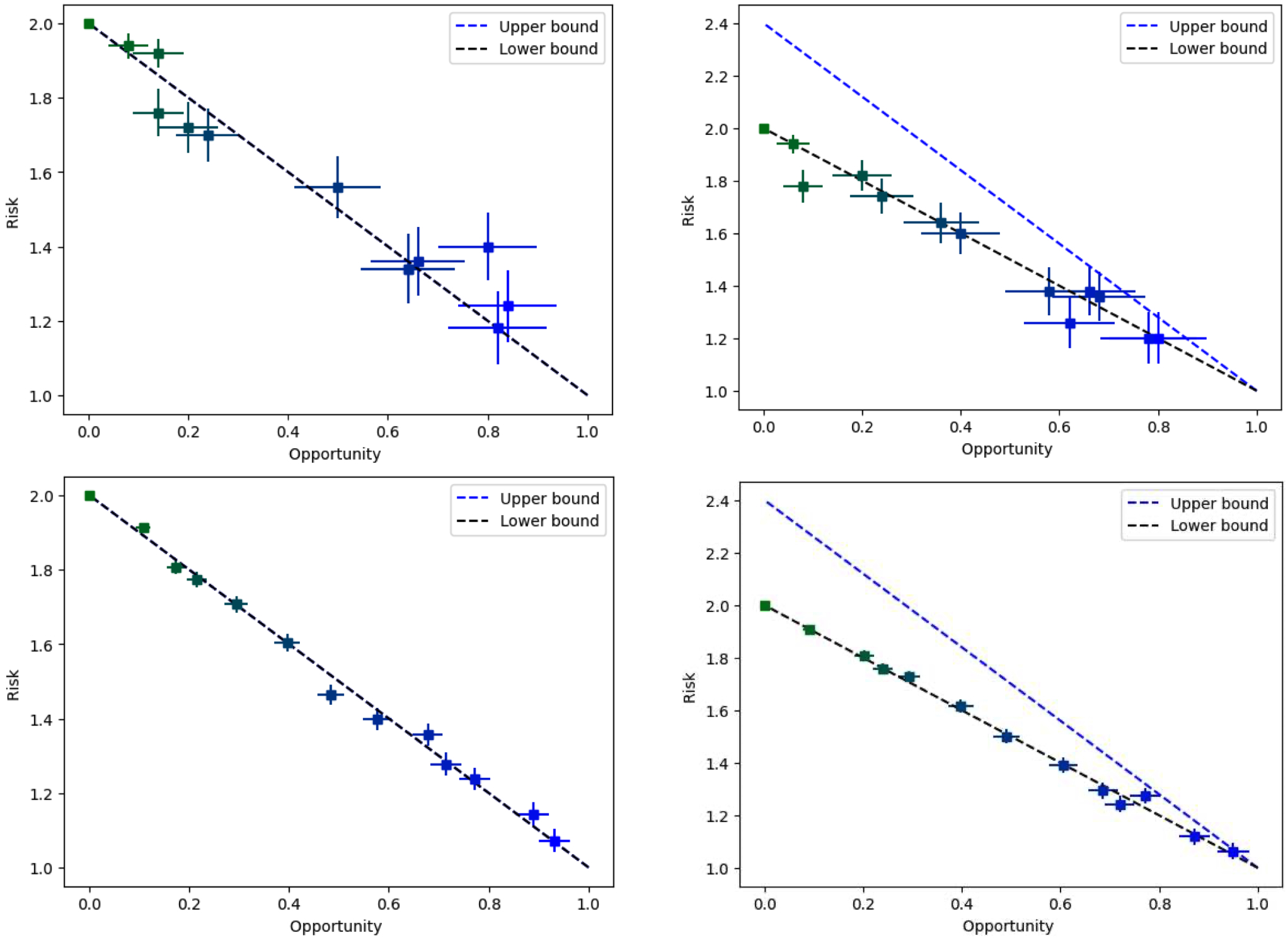} 
    \caption{Opportunity-risk tradeoff for Matching Pennies (MP) and Adjusted Matching Pennies (AMP) games over 100 (top) and 1,000 (bottom) runs. \textbf{Left}: Matching Pennis. \textbf{Right}: Adjusted Matching Pennies. As the number of runs increases, the variance of both risk and opportunity values decreases.}
    \label{fig:bounds}
\end{figure}
\newpage

\subsection{Details on Green Security Game simulations}
\subsubsection{Data description}
The data is from a study tracking geo-location information of 32 African elephants in the Hwange National Park in Zimbabwe, Africa. The data was collected from 2009 to 2017 and includes time stamp information as well as longitudinal and latitudinal information of the elephants
\subsubsection{Data processing}
We divided each year into two based on the seasons in Zimbabwe which would affect the location of the elephants as they migrate given changes in the environment. A year was divided into a ``Rainy'' season which runs from October to March and a  ``Dry'' season which runs from April to September.

We divided the area covered in the dataset into 9 locations and thus had a $3 \times 3$ grid which served as a surveillance area. For each elephant in the season, we calculated its mean location and so for each of the seasons we had a mean location for the elephants. It is important to note, given migratory nature of elephants across national borders, each season does not necessarily have all 32 elephants as they move as a result of a wide range of factors (e.g., availability of water). Some seasons, also notably, do not have any recorded elephant presence in the area under surveillance. We do not see this as a limitation as the dynamic environment presents changes in the game which the defender has to take into account. 

\subsubsection{Green Security Game}
Each of the seasonal elephant information presents us with a state for our Stochastic Bayesian Game. We define the game in the following manner:
\begin{itemize}
    \item \textbf{State:} We have 16 states from the seasonal information we get from the data
    \item \textbf{Actions:} Each player (attacker or defender) has 9 available actions each corresponding to a selection of an area in the $3 \times 3$ grid defined above. 
    \item \textbf{Payoffs:} Let $n$ be the number of elephants recorded in a particular grid square: 
    \begin{itemize}
        \item If the attacker and defender select the same grid square, the defender gets $n$ while the attacker gets $-2$
        \item If the attacker and defender select different grid squares, let $n_{att}$ be the number of elephants in the grid square selected by the attacker. The defender gets payoff $-n_{att}$ while the attacker gets payoff $n_{att}$
    \end{itemize}
    These payoffs were designed to model the asymmetric nature of the defending task. An offender often will get a fixed penalty as stipulated by law, whilst the defender always depends on the number of elephants the attacker has access to. 
    \item \textbf{Transitions:} To take into account the effect of changes in weather whilst also bringing in some stochasticity, we made it such that there is some transition probability between any ``Rainy'' state to any ``Dry'' state and vice-versa. We do however, slightly, increase the probability of transition between adjacent historical states, to reflect historical data. We have the probability of transitioning between any two ``Rainy'' or ``Dry'' states to be zero. Our transition probabilities do not depend on the actions taken by the agents. 
\end{itemize}
\subsubsection{Type definitions}
We make use of the same types as in the $2 \times 2$ game simulations. We make adjustments to the Markovian types in that Types 1 and 2 now select the grid with the highest population of elephants and second highest population of elephants, respectively. The Leader-Follower-Trigger Agents now looks to see if the other player is selecting the grid with the highest number of elephants for more than $50\%$ of the historic play in which case they turn to play their minimax strategy

\subsection{Additional $2 \times 2$ game evaluations}
We include as an illustration as well as for completeness a couple of other $2 \times 2$ games we also evaluated against. This is helpful as it shows the variety of tradeoffs that exist within the topology of $2 \times 2$ games. In particular, we see some games exhibiting gradation as the agent moves from fully robust to fully trusting whilst in others there does not exist such a tradeoff because of the existence of a dominant strategy for the row player regardless of the type of the column player (e.g., the last game we show in this section). We note that the games provided in this file do not exhaust the entire topology of $2\times 2$ games. They are however added to show the range of tradeoffs that could exist in games. 

\newpage
\begin{figure}[h!]
\centering
\includegraphics[width=0.75\textwidth]{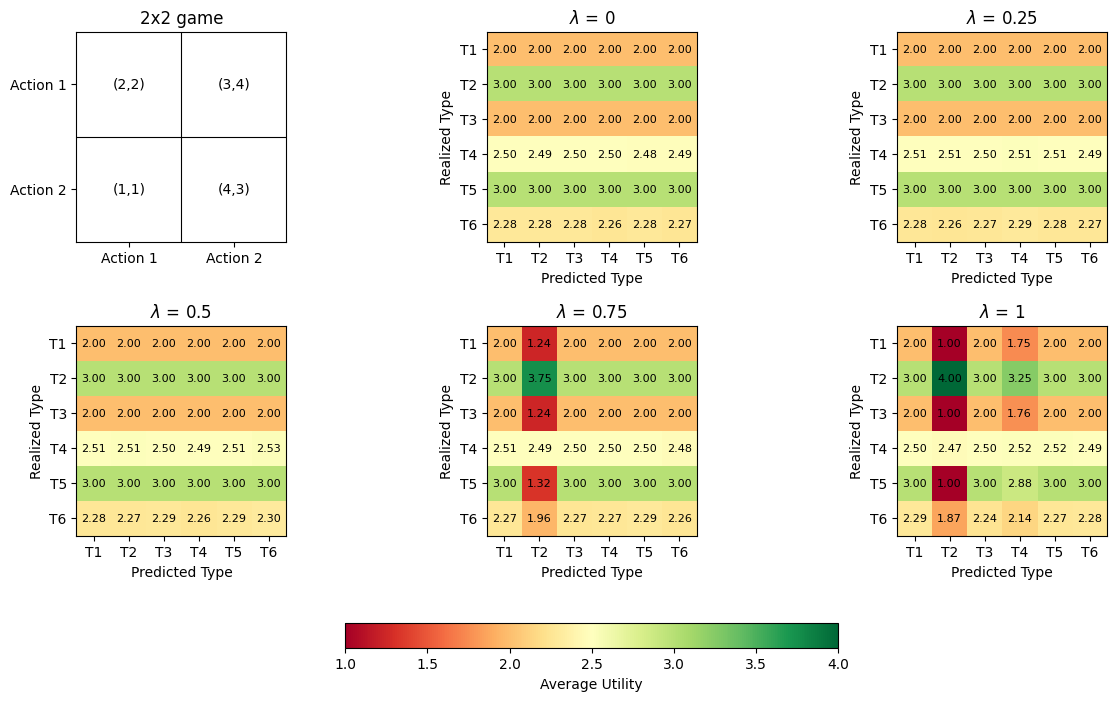}
\end{figure}
%%%%%%%%%%
\begin{figure}[h!]
\centering
\includegraphics[width=0.75\textwidth]{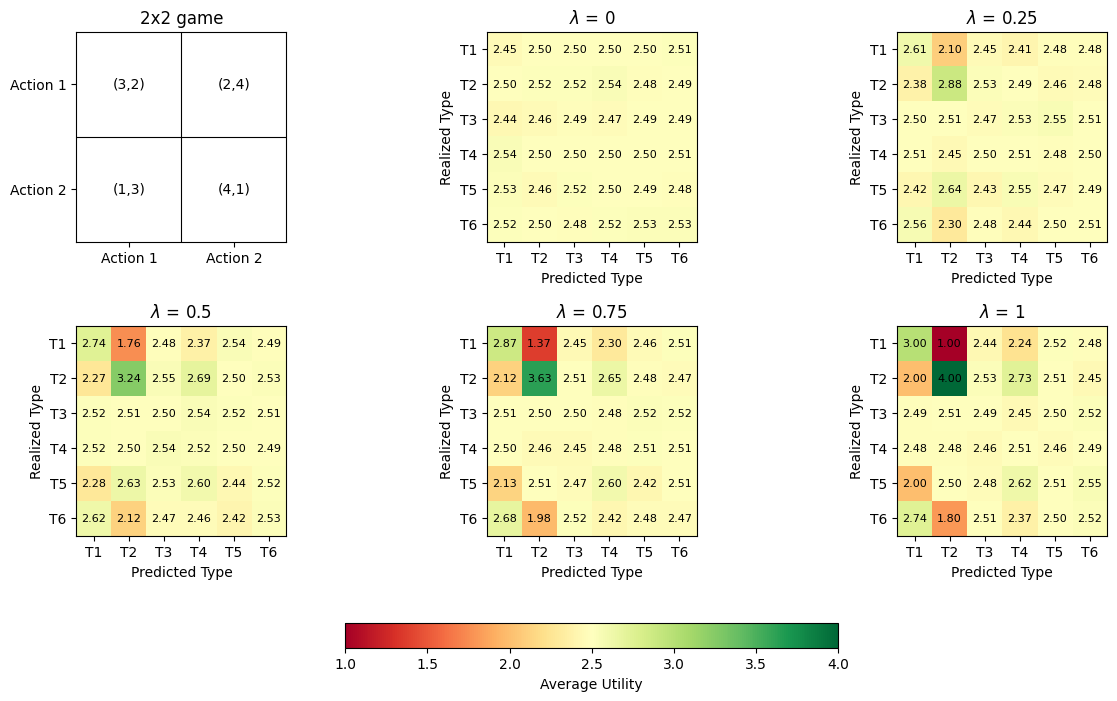}
\end{figure}
%%%%%%%%%%
\begin{figure}[h!]
\centering
\includegraphics[width=0.75\textwidth]{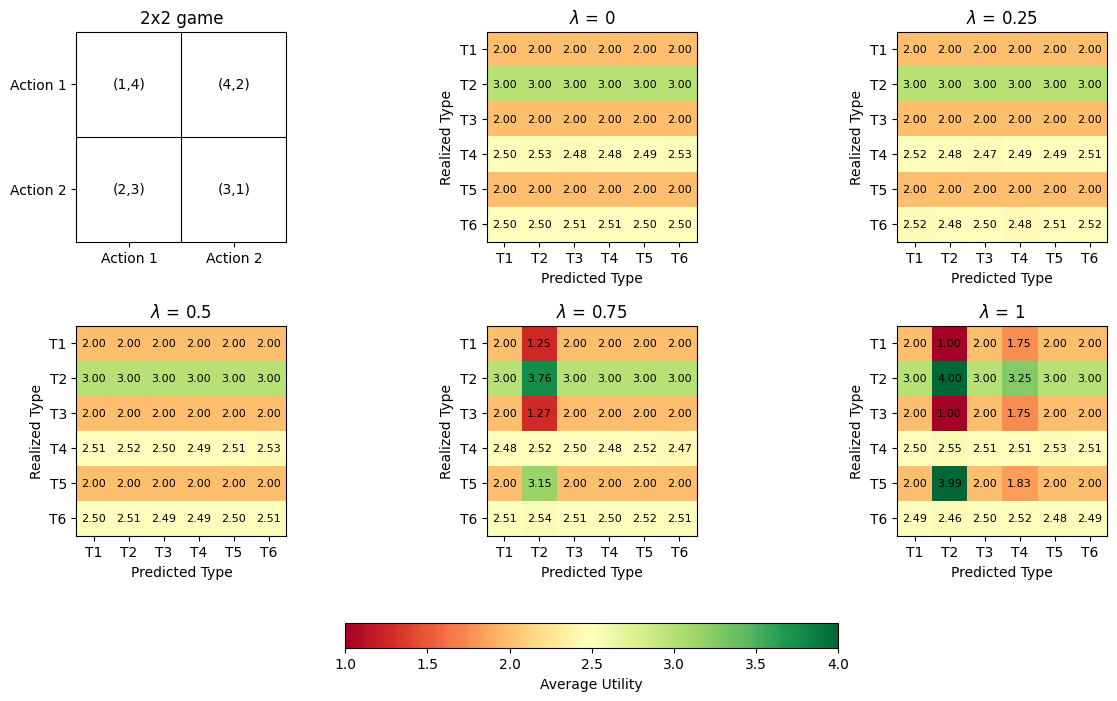}
\end{figure}
%%%%%%%%%%
\begin{figure}[h!]
\centering
\includegraphics[width=0.75\textwidth]{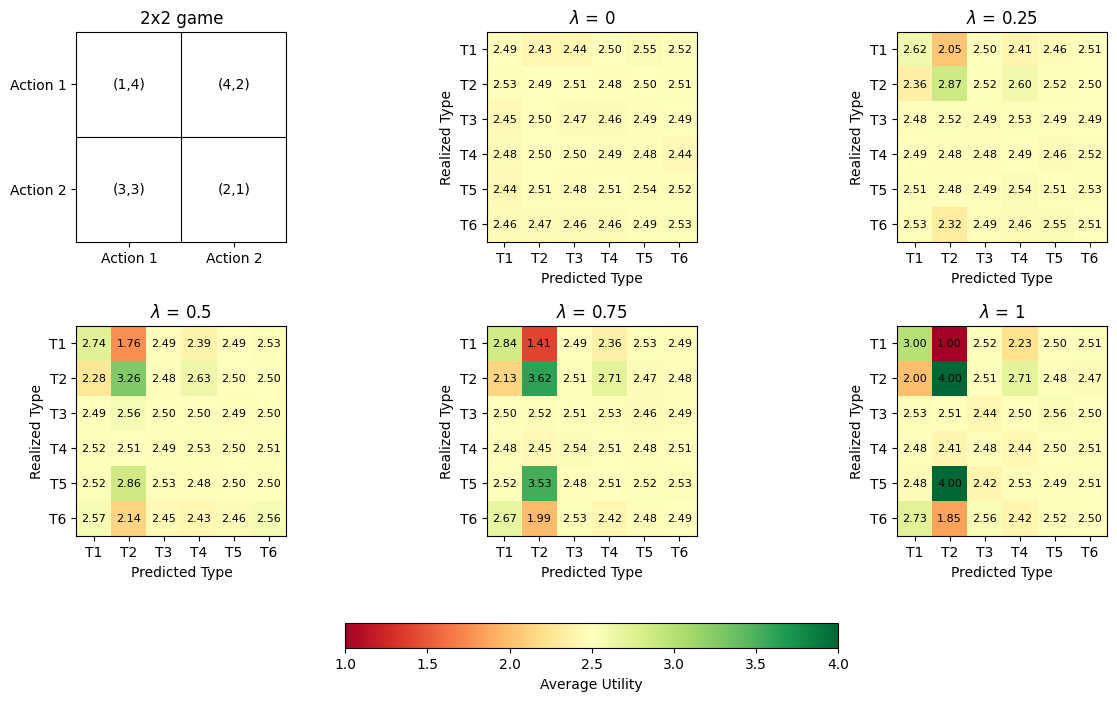}
\end{figure}
%%%%%%%%%%
\begin{figure}[h!]
\centering
\includegraphics[width=0.75\textwidth]{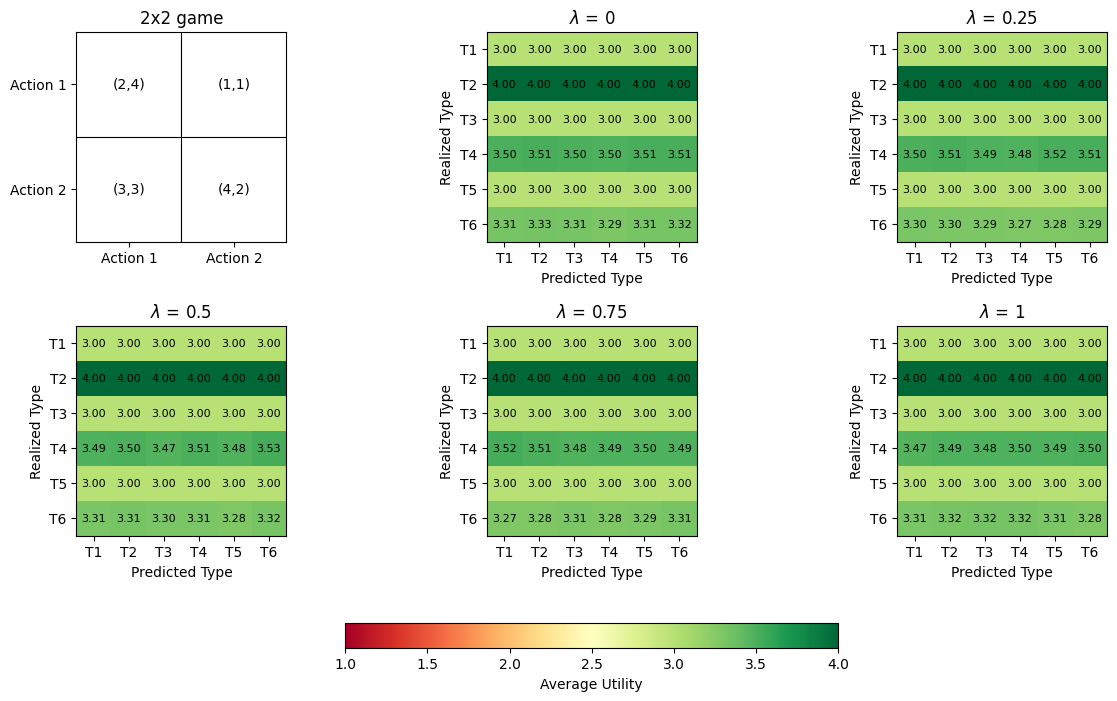}
\end{figure}

\end{document}